\documentclass[twoside]{article}

\usepackage[accepted]{aistats2026}
\usepackage{breakcites}
\usepackage[utf8]{inputenc} % allow utf-8 input
\usepackage[T1]{fontenc}    % use 8-bit T1 fonts
\usepackage{hyperref}       % hyperlinks
\usepackage{url}            % simple URL typesetting
\usepackage{booktabs}       % professional-quality tables
\usepackage{amsfonts}       % blackboard math symbols
\usepackage{nicefrac}       % compact symbols for 1/2, etc.
\usepackage{microtype}      % microtypography
\usepackage{xcolor}         % colors
\usepackage{comment}        % comments
\usepackage{amsmath,amssymb,amsthm,dsfont,bm}
\usepackage{array}
\usepackage{extarrows}
\usepackage{wrapfig}        % surround picture
\usepackage{tikz}
\usetikzlibrary{arrows.meta, positioning, shapes.symbols}
\usepackage{pifont}
\usepackage{graphicx}
\definecolor{darkgreen}{rgb}{0.0, 0.62, 0.0}
\newcommand{\circletext}[1]{%
  \tikz[baseline=(char.base)]{
    \node[draw, shape=circle, inner sep=1pt] (char) {#1};
  }%
}
\usepackage[most]{tcolorbox}
\usepackage{enumitem}
\tcbuselibrary{breakable}

\makeatletter
\let\oldsection\section
\renewcommand{\section}{\@ifstar\oldsection\SectionUpper}
\newcommand{\SectionUpper}[1]{\oldsection{\MakeUppercase{#1}}}
\makeatother

\newtheorem{theorem}{Theorem}
\newtheorem{lemma}{Lemma}
\newtheorem{proposition}{Proposition}
\newtheorem{corollary}{Corollary}
\newtheorem{definition}{Definition}
\newtheorem{assumption}{Assumption}

\newtheorem{remark}{Remark}
\newtheorem{example}{Example}
\newtheorem{condition}{Condition}
\newtheorem*{condition*}{Condition}

\DeclareMathOperator*{\argmax}{arg\,max}
\DeclareMathOperator*{\argmin}{arg\,min}

\newcommand{\N}{\mathbb{N}}
\newcommand{\R}{\mathbb{R}}

\usepackage{mathtools}

\newcommand{\CAL}[1]{\mathcal{#1}}

\newcommand{\zyh}[1]{{\color{blue}{ [\textbf{Youheng:}} #1]}}
\newcommand{\yplu}[1]{{\color{orange}{ [\textbf{Yiping:}} #1]}}
\begin{document}
% \showthe\textwidth
% \showthe\columnwidth
% \showthe\oddsidemargin
% \showthe\baselineskip



\twocolumn[

\aistatstitle{A Covering Framework for Offline POMDPs Learning using Belief Space Metric}

\aistatsauthor{ Youheng Zhu \And Yiping Lu }

\aistatsaddress{ Northwestern University \And Northwestern University } ]

\begin{abstract}
In off‑policy evaluation (OPE) for partially observable Markov decision processes (POMDPs), an agent must infer hidden states from past observations, which exacerbates both the curse of horizon and the curse of memory in existing OPE methods. This paper introduces a novel covering analysis framework that exploits the intrinsic metric structure of the belief space (distributions over latent states) to relax traditional coverage assumptions. By assuming value-relevant functions are Lipschitz continuous in the belief space,%\yplu{true?}
we derive error bounds that mitigate exponential blow-ups in horizon and memory length. Our unified analysis technique applies to a broad class of OPE algorithms, yielding concrete error bounds and coverage requirements expressed in terms of belief space metrics rather than raw history coverage. We illustrate the improved sample efficiency of this framework via case studies: the double sampling Bellman error minimization algorithm, and the memory-based future-dependent value functions (FDVF). In both cases, our coverage definition based on the belief‐space metric yields tighter bounds.
\end{abstract}

%\yplu{title: Unified analysis for Non-parametric POMDPs using Belief Space Smoothness you may also try \url{https://arxiv.org/pdf/2011.01797}}
\section{Introduction}

%Off-policy evaluation (OPE) is an important topic in offline reinforcement learning, of which the target is to estimate the expected cumulative reward of a policy $\pi_e$. The data used to estimate $\pi_e$ comes from a different behavior policy $\pi_b$, since real world agents generally cannot interactively collect data from the environment, but to learn from offline data. In this paper,  we adopt a setting where people can only observe partial information from the latent states of a latent MDP, thus the observation dynamics becomes non-Markovian, and could depend on the entire history of action-observation pairs. Such model, named the partially observable MDP (POMDP), is more powerful in modeling real world problems, yet faces more challenges than common MDPs.
Off-policy evaluation (OPE) is a central problem in offline reinforcement learning, aiming to estimate the expected cumulative reward of a target policy $\pi_e$ using data collected under a different behavior policy $\pi_b$. This setting arises naturally in real-world applications, where interactive data collection is often impractical or unsafe, and learning must rely solely on pre-collected offline trajectories. In this paper, we consider a more realistic yet challenging setting where only partial observations of the underlying latent states are available. This leads to non-Markovian observation dynamics that may depend on the entire history of action-observation pairs. Such scenarios are modeled by partially observable Markov decision processes (POMDPs), which offer greater expressiveness for real-world problems~\cite{atrash2009development,lauri2022partially} but introduce significant complexity compared to fully observable MDPs.

%\yplu{first paragraph first introduce Off-policy evaluation (OPE) aims to evaluate a target policy \pi_e using an offline dataset collected by a different behavior policy \pi_b. The problem plays a crucial role in reinforcement learning (RL), and is particularly relevant to real-world scenarios where policies need to be properly evaluated before online deployment. In this paper, we study OPE in non-Markov environments modelled as partially observable MDPs (POMDPs),1 where the observation space is large and demands the use of function approximation. In POMDPs, the agent only has access to observations rather than the latent state, and the next observation may depend on the entire history of observation-action sequences (or simply, the history). A common approach to apply MDP techniques is to treat the history as the state, thereby reducing a POMDP to a history-based MDP. follow \url{https://openreview.net/pdf?id=Qja5s0K3VX}}

Although for a POMDP, Markovian is restored when treating history trajectories as states, in which case the POMDP is reduced to a MDP problem, directly applying conventional MDP methods, such as Importance Sampling and Bellman residual minimization, inevitably leads to error bounds exponentially scaling with horizon $H$, a phenomenon termed the \textit{curse of horizon}. For instance, in importance sampling, the sequential importance weights grow exponentially with the horizon, leading to an intractable variance in the estimation. To alleviate this issue, a method called the Future Dependent Value Function (FDVF) is proposed for memoryless policies but fail when memory-based policies are introduced, in which case the coverage scales exponentially with the memory length, facing the \textit{curse of memory} \cite{zhang2024cursesfuturehistoryfuturedependent}.

%\yplu{second paragraph introduce what is curse of horizon, what is curse of memory, just change the order}

%\yplu{this is the third pargraph, then  you have enough inofrmation for asking the questions}
To overcome the curses of horizon and memory in history‑as‑state MDPs, we reformulate the problem in the \emph{belief space}, a central concept in POMDPs, defined as the space of probability distributions over latent states given the observed history of actions and observations. Each element in belief space—referred to as a belief state—serves as a proxy for historical trajectories. As an explicit computation of a belief state requires the latent dynamic to be transparent to the agent, it is most commonly used in POMDP planning literature. Utilizing the metric structure of belief spaces, planning methods like point-based value iteration (PBVI) achieve efficient solutions by sparsely covering belief subspaces \cite{shani2013survey, lee2007makes, zhang2014covering}. 
Although belief-space structure has been extensively studied in POMDP planning~\cite{lee2007makes} and model learning~\cite{zhang2012covering}, its role in off-policy evaluation (OPE) remains largely underexplored.
%In~\citep{zhang2012covering}, an algorithm that utilizes covering is proposed for learning POMDP model, showcasing the potential of covering as a hardness measure for POMDP planning and learning.
 Notably, current offline learning approaches typically neglect this metric structure, treating history spaces explicitly, resulting in exponential dependence on the horizon length.
%\zyh{An issue about introducing OPE in the second paragraph is, it's too long and readers cannot tell its relation with the POMDP planning literature, which is our motivation. My idea is, can we first claim that there is an exponential problem with POMDP learning currently without going into details, and throw out the question first. Then talk about OPE and what the curse of horizon/memory is. Like our current structure?} 
This raises a critical question:

%\textit{Belief space's metric has been proven good for measuring computational complexity in planning, how about sample complexity in offline learning? In other words, can we leverage belief space metric structures to avoid the exponential complexity in offline POMDP learning?}
\textit{While the metric structure of belief space has proven effective for characterizing computational complexity in POMDP planning, can it similarly characterize sample complexity in offline learning? More specifically, can we exploit this belief space structure to circumvent exponential complexity in offline POMDP learning?}

%\yplu{you can have two question belief space is good for planning how can it helps the sample complexity?}

%Specifically, we focus on Off-Policy Evaluation (OPE), estimating cumulative rewards for a target policy $\pi_e$ using data from a distinct behavior policy $\pi_b$. Directly treating history trajectories as states in conventional MDP methods, such as Importance Sampling and Bellman residual minimization, inevitably leads to error bounds exponentially scaling with horizon $H$, a phenomenon termed the \textit{curse of horizon}. For instance, in importance sampling, the sequential importance weights grow exponentially with the horizon, leading to an intractable variance in the estimation. To alleviate this issue, a method called the Future Dependent Value Function (FDVF) is proposed for memoryless policies but fail when memory-based policies are introduced, in which case the coverage scales exponentially with the memory length, facing the \textit{curse of memory} \cite{zhang2024cursesfuturehistoryfuturedependent}. %\yplu{use natural language to discribe the motivation of the curse of horizon and curse of memory here. e.g. collect all the importance sampling weight through the horizon thus leads to variance exp to the horizon..}\zyh{\ding{51}}
%Motivated by this question, our work explores the metric properties of belief spaces, aiming for efficient guarantees even under exponentially large history space, or with memory-based policies.  %\yplu{follow my abstract's logic and extend everything, cite hte table}

\paragraph{Our Contributions}
Motivated by this question, our work explores the idea of belief space metric structure, and studies the theoretical guarantees of some common model-free OPE algorithms using belief metric. The core idea of our framework is similar to that of state abstraction~\cite{li2006towards_abstraction}, given that the complexity of belief space can be lowered significantly through an abstraction that contracts similar states. That is to say, if two history trajectories have similar belief states, they should be considered similar in the analysis.
To do this, we restrict ourselves to a \textbf{subset of policies}, i.e. the policies with \emph{stability}. This structural assumption on policy class is \emph{rich enough} to contain all possible policies of our interest, and possesses nice properties for tighter coverage.
%Consequently, it remains to dominate the error caused by the abstraction, which is also presented in this paper. 
The overall result of our analysis with comparison to existing results that suffer from the curse of horizon/memory is presented in Table~\ref{table:comparison} below. In general, our result mitigates the exponentiality of coverage especially under smoothness structure of belief space as shown in Example~\ref{example:covering_num_H_infty} and~\ref{example:Fast_forgetting_H_infty}. 
%However, when the belief states spread sparsely in the belief space, for example when every history has a distinct one-hot belief, the POMDP is merely equivalent to a MDP with exponential state space, and the belief metric becomes a discrete metric, thus trivializes our analysis.%\zyh{add some limitation and when our result is most useful.}
  To specify our contributions:
\begin{itemize}[itemindent=0.5cm,leftmargin=0.2cm,topsep=0pt]
\setlength{\itemsep}{0pt}
\setlength{\parsep}{0pt}
\setlength{\parskip}{0pt}
    \item We propose a framework of analysis that uses state abstraction induced by $\varepsilon$-covering to obtain a coverage on the abstract space, which adapts to a wide range of scenarios in the OPE problem. This framework easily generalizes to other algorithms or even other reinforcement learning tasks.
    \item We show in Table~\ref{table:comparison}, Theorem~\ref{Thm:coverage_compare_2} and~\ref{Thm:coverage_compare_3} that our coverage obtained using belief space covering is no worse than the original coverage. We also show in Example~\ref{example:covering_num_H_infty} and~\ref{example:Fast_forgetting_H_infty} that our coverage resolves the curse or horizon/memory under specific smoothness property of the POMDP model.
    \item In Chapter~\ref{Chapter:double_sampling}, we complete a detailed analysis specifically for double sampling algorithm as an example of Bellman error minimization algorithms. In Chapter~\ref{Chapter:FDVF}, we also present the pipeline on future dependent value function where the fast forgetting properties of POMDP and policies are adopted. We then show that FDVF admits a simpler analysis, free from any assumptions on the POMDP system itself. \textbf{This indicates that the "curse of memory" is much easier to handle than the "curse of horizon"}. Additionally, this provide an answer to the question left by~\cite{zhang2024cursesfuturehistoryfuturedependent}, that with structural assumption on the policy, we can mitigate the "curse of memory".
    %We justify the Lipchitz assumption for policy and value function.\zyh{todo}
\end{itemize}

\section{Related Works}

\paragraph{POMDP planning.}
In POMDP planning literature, the idea of point-based value iteration (PBVI)~\cite{shani2013survey, inproceedings,inproceedings-2,4667695,article,10.5555/1622519.1622525} is to computes on a finite subset of the entire belief space, aiming for an optimal policy. Notably, an important characteristic of PBVI is that its selection of belief subspace uses the metric structure in belief space, namely, every time the algorithm expands the belief subset, it searches for the furthest belief point w.r.t. the current belief subset that is one-step reachable, so that the reachable belief subset can be constructed as sparse as possible. Additionally, the connection between complexity and belief space metric was identified by~\cite{lee2007makes,zhang2014covering}, which proved the existence of approximate algorithm with complexity polynomial to the covering number of reachable belief space. %\yplu{can't understand, follow\cite{lee2007makes} and the papers that cite this one}

\paragraph{Curse of Horizon and Curse of Memory in OPE.}
Numerous algorithms have addressed Off-Policy Evaluation (OPE) in fully observable MDPs, such as Importance Sampling \cite{precup2000eligibility, jiang2016doubly, jiang2019entropy, hu2023off}, Fitted Q-Iteration (FQE) \cite{article-1,article-2,pmlr-v97-le19a}, Bellman residual minimization with double sampling \cite{BAIRD199530}, min-max estimators \cite{inproceedings-3,pmlr-v97-chen19e,NEURIPS2019_48e59000,unknown-4,uehara2021finite,NEURIPS2022_14ecbfb2}, and marginalized importance sampling \cite{pmlr-v119-uehara20a}. However, directly applying these approaches to Partially Observable MDPs (POMDPs) by treating each trajectory history as a distinct state encounters a fundamental challenge known as the \textit{curse of horizon}: the error bounds become exponentially worse as the horizon grows, due to coverage assumptions expanding with the exponentially large history space. Alternatively, recent approaches such as the Future Dependent Value Function (FDVF) \cite{NEURIPS2023_3380e811, zhang2024cursesfuturehistoryfuturedependent} address this by shifting coverage requirements onto latent states, thus providing polynomial guarantees for memoryless policies. Nevertheless, this method is constrained by the \textit{curse of memory}, as its complexity reverts to exponential when extended to memory-based policies, due to the necessity of capturing dependencies between future observations and historical memory states, dramatically increasing coverage complexity.

\begin{table*}[h!]
\caption{Comparison With Existing Coverage} %\yplu{what is $T(n^{-1})$}}\zyh{that's $T$'s dependency on $n^{-1}$, perhaps I'll just put $T$.}
\label{table:comparison}
\begin{center}
\begin{tabular}{>{\centering\arraybackslash} m{4.2cm}|>{\centering\arraybackslash} m{3.6cm}|>{\centering\arraybackslash} m{0.7cm}|>{\centering\arraybackslash} m{0.7cm}|>{\centering\arraybackslash} m{4.7cm}}
%\begin{tabular}{|Sc|Sc|Sc|}

\hline\hline
\textbf{Criteria} & \multicolumn{2}{>{\centering\arraybackslash} m{4.3cm}|}{\textbf{Existing Coverage With Curse of Horizon/Memory}} & \multicolumn{2}{>{\centering\arraybackslash} m{5.4cm}}{\textbf{Our Coverage using Belief Space Smoothness}} \\
\hline 
\multicolumn{5}{c}{Bellman Error Minimization (\emph{e.g.} Double Sampling )} \\
\hline\hline
Coverage Definition~\cite{jiang2024offline_survey} & 
\multicolumn{2}{c|}{$\displaystyle \bigg\|\frac{d^{\pi_e}({\color{red}\tau_h},a)}{d^D({\color{red}\tau_h},a)}\bigg\|_\infty$}
 & \multicolumn{2}{c}{$\displaystyle \bigg\|\frac{d^{\pi_e^\phi}({\color{darkgreen}\phi(b)},a)}{d^D({\color{darkgreen}\phi(b)},a)}\bigg\|_\infty$} \\
\hline
Coverage Worst Case\footnotemark[1] Scale & 
$\color{red}\displaystyle |\CAL{B}|= \Theta((|\CAL{O}||\CAL{A}|)^H)$
&
 \multicolumn{2}{c|}{>}
 & $\color{darkgreen}\displaystyle \mathrm{Covering}(\CAL{B}, \Theta(n^{-1/2}))$\footnotemark[2] \\
\hline
Ability to handle $H\to\infty$ & \multicolumn{2}{c|}{{\color{red} \ding{56}}: Infinite} & \multicolumn{2}{c}{{\color{darkgreen} \ding{51}}: Polynomial guarantee see example~\ref{example:covering_num_H_infty}}\\
\hline
\multicolumn{5}{c}{Future Dependent Value Function}\\
\hline\hline
Coverage Definition~\cite{zhang2024cursesfuturehistoryfuturedependent} & \multicolumn{2}{c|}{${\sup_{h,V}\sqrt{\frac{\mathbb{E}_{\pi_e}[(\CAL{B}^{(\CAL{S},\CAL{H}_H)}V)(s_h,{\color{red}\tau_{h}})^2]}{\mathbb{E}_{\pi_b}[(\CAL{B^H}V)(\tau_h)^2]}}}$} & \multicolumn{2}{c}{${\sup_{h,V}\sqrt{\frac{\mathbb{E}_{\pi_e^\phi}[(\CAL{B}^{(\CAL{S},\CAL{H}_T)}V)(s_h,{\color{darkgreen}\tau_{[h-T+1:h]}})^2]}{\mathbb{E}_{\pi_b^\phi}[(\CAL{B^H}V)(\tau_h)^2]}}}$} \\
\hline
$L_2$ Belief Coverage (One-hot Belief)~\cite{zhang2024cursesfuturehistoryfuturedependent} & $\displaystyle {   \mathbb{E}_{\pi_b}\bigg[\bigg(\frac{d^{\pi_e}(s_h,{\color{red}\tau_h})}{d^{\pi_b}(s_h,{\color{red}\tau_h})}\bigg)^2\bigg]    }$ 
&
 \multicolumn{2}{c|}{\shortstack{> \\Theorem~\ref{Thm:coverage_compare_2}}}
& $\displaystyle {   \mathbb{E}_{\pi_b^\phi}\bigg[\bigg(\frac{d^{\pi_e^\phi}_\phi(s_h,{\color{darkgreen}\tau_{[h-T+1:h]}})}{d^{\pi_b^\phi}_\phi(s_h,{\color{darkgreen}\tau_{[h-T+1:h]}})}\bigg)^2\bigg]    }$\\
\hline
$L_\infty$ Belief Coverage (One-hot Belief)~\cite{zhang2024cursesfuturehistoryfuturedependent} & $\displaystyle {   \bigg\|\frac{d^{\pi_e}(s_h,{\color{red}\tau_h})}{d^{\pi_b}(s_h,{\color{red}\tau_h})}\bigg\|_\infty    }$ 
&
 \multicolumn{2}{c|}{\shortstack{> \\Theorem~\ref{Thm:coverage_compare_3}}}
& $\displaystyle {\   \bigg\|\frac{d^{\pi_e^\phi}_\phi(s_h,{\color{darkgreen}\tau_{[h-T+1:h]}})}{d^{\pi_b^\phi}_\phi(s_h,{\color{darkgreen}\tau_{[h-T+1:h]}})}\bigg\|_\infty    }$\\
\hline
$L_\infty$ Worst Case\footnotemark[1] (One-hot Belief) & $\color{red} \Theta((|\CAL{O}||\CAL{A}|)^H)$ 
&
 \multicolumn{2}{c|}{>}
& $\color{darkgreen} \Theta((|\CAL{O}||\CAL{A}|)^{T})$ \\
\hline
Ability to handle $H\to\infty$\footnotemark[3] & \multicolumn{2}{c|}{{\color{red} \ding{56}}: Infinite} & \multicolumn{2}{c}{{\color{darkgreen} \ding{51}}: Polynomial guarantee see example~\ref{example:Fast_forgetting_H_infty}}\\
\hline
\end{tabular}
\end{center}

\end{table*}

\section{Preliminaries}\label{sect:prelim}
\begin{comment}
In this section, we introduce the model setup and algorithms for off-line POMDPs.
\yplu{You'd better first try to have a single paragraph try to summarize your contribution. What's the common thing previous work is doing? Why that's bad? How you do that? comparison with preivous works. add citation. }

\yplu{also do a literature review for POMDP}
\begin{itemize}
    \item cofounded POMDP:  This is because these methods require the behavior policy to only
depend on the latent state to ensure certain conditional independence assumptions, \yplu{your paper is mostly lies in \url{https://arxiv.org/pdf/2207.13081}? read all the paper cites the paper to see the literature?}
\end{itemize}
\end{comment}

\paragraph{Infinite-horizon Discounted POMDP:}
An infinite-horizon discounted POMDP can be specified as a 7-tuple: $\mathcal{M}=\langle \mathcal{S},\mathcal{A},\mathcal{O},r,\gamma,\mathbb{O},\mathbb{T} \rangle$ where $\gamma\in [0,1)$ is the discount factor, $\mathcal{S}$ is the latent state space, %\yplu{should we write $S=\cup_{h=1}^H S_h$  where $S_h$ is the state space at step $h$, the latent here is vague}\zyh{I think it's more appropriate to just let each $S_h$ be identical, so that using $S$ is fine.} 
$\mathcal{A}$ is the action space, $\mathcal{O}$ is the observation space, $r:\mathcal{S}\times \mathcal{A}\to [0,R_{\rm max}]$ is the bounded reward function, $\mathbb{O}:\mathcal{S}\to\Delta(\mathcal{O})$ is the emission kernel (i.e., the conditional distribution of the observation given the state), and $\mathbb{T}:\mathcal{S}\times\mathcal{A}\to\Delta(\mathcal{S})$ is the transition kernel (i.e., the conditional distribution of the next state given the current state-action pair). We use $\Delta(\cdot)$ to represent probability distributions on the given space, and $|\cdot|$ for the cardinality of a set. For simplicity and without loss of generality, we assume discrete and finite spaces $\mathcal{S},\mathcal{A},\mathcal{O}$, of which the cardinality can be large.

The POMDP evolves as follows: starting from an initial latent state $s_1\sim d_0(s)$, at each step $h$, the latent state $s_h$ emits an observation $o_h$ drawn from $\mathbb{O}(s_h)$, and the environment generates a reward $r_h$ based on the current state-action pair $(s_h,a_h)$. The state then transitions according to $s_{h+1}\sim\mathbb{T}(s_h,a_h)$. Crucially, in general POMDPs, the learner has no access to the latent state space $\CAL{S}$; instead, only trajectories collected under an offline behavior policy are available.

%The system dynamic can be uniquely demonstrated by the following procedure: Starting from a latent state $s_1\sim d_0(s)$, the latent MDP $\langle \CAL{S},\CAL{A},r,\gamma,\mathbb{T} \rangle$ evolves under its latent dynamic $\mathbb{T}$. At step $h$, the state emits an observation $o_h$ automatically according to $\mathbb{O}$, and generates a reward $r_h$ based on the action $a_h$ received. %The action $a_h$ is decided by a policy $\pi:\CAL{H}\to\Delta(\CAL{A})$
%It is important to note that in a general POMDP, the latent state space $\mathcal{S}$ is unknown to the learner, and learners only get access to the trajectories sampled by a behavior policy, which is invariant under the offline setting. 

We also consider the finite-horizon POMDP setting extensively discussed in Chapter~\ref{Chapter:FDVF}. In the finite-horizon scenario, we set the discount factor $\gamma=1$, and the agent interacts with the environment for a finite number of steps $H$. 

%It is also convenient for us to introduce the finite-horizon setting for POMDPs, which is majorly talked about in chapter~\ref{Chapter:FDVF}. In the finite horizon settings, the discounted factor $\gamma$ is set to $1$, and the agent stops at step $H$. \yplu{follows one paper's language to modify and use chatgpt to refine.}

\paragraph{Offline Data:}
The offline dataset $\mathcal{D}$ is collected using a behavior policy $\tilde{\pi}_b$. The process involves independently collecting $n$ sample trajectories $(o_1, a_1, \cdots)$ from the POMDP. From each trajectory, a prefix of the first $h$ elements is truncated to form a tuple  $(o_1,a_1,r_1,o_2,a_2,r_2,\cdots, o_h,a_h,r_h,o_{h+1})$ where $h$ is randomly selected. Finally, the dataset takes the form of $\CAL{D}_1$ as shown below. In chapter 6, for the future-dependent value function (FDVF), the definition of offline data differs slightly. In the FDVF setting, we consider a finite-horizon POMDP of length $H$.
Again, a behavior policy $\pi_b$ is used to interact with the environment and collect data. This time, the entire trajectory is treated as a single data point, as shown by $\CAL{D}_2$.
\begingroup
\allowdisplaybreaks
\begin{align*}
\mathcal{D}_1=&\;\{(o_1^{[i]},a_1^{[i]},r_1^{[i]},\cdots, o_{h_i}^{[i]},a_{h_i}^{[i]},r_{h_i}^{[i]},o_{{h_i}+1}^{[i]})\}_{i=1}^n,\\ 
    \mathcal{D}_2&=\{((o_1^{[i]},a_1^{[i]},r_1^{[i]},\cdots,o_{{H}}^{[i]},a_{H}^{[i]},r_{H}^{[i]})\}_{i=1}^n
\end{align*}
\endgroup

\paragraph{State Abstraction:} For a MDP $(\CAL{S},\CAL{A},r,\gamma,P)$ where $P:\CAL{S}\times\CAL{A}\to\Delta(\CAL{S})$ denotes the transition kernel, an abstraction $\phi$ is a mapping from $\CAL{S}$ to an abstract state space $\CAL{S}_\phi$, and the MDP is transformed into an abstract MDP $(\CAL{S}_\phi,\CAL{A},r_\phi,\gamma,P_\phi)$ where $r_\phi(\phi(s),a):=\mathbb{E}_{s'\sim p_{\phi(s)}}[r(s',a)]$ and $P_\phi(\phi(s_d)|\phi(s),a):=\mathbb{E}_{s'\sim p_{\phi(s)}}[\sum_{\phi(s'')=\phi(s_d)}P(s''|s',a)]$. Here $\{p_x\}_{x\in\CAL{S}_\phi}$ is any family of distributions in which $p_x$ being supported on $\phi^{-1}(x)$. For any function defined on the abstract system $f_{\rm bin}:\CAL{S}_\phi\to\R$, we define the lifted version of which as $[f_{\rm bin}]_{\rm true}(\cdot):=f_{\rm bin}(\phi(\cdot))$. Similar for an abstract policy $\pi_\phi:\CAL{S}_\phi\to\Delta(\CAL{A})$, of which the lifted version $[\pi_\phi]_{\rm true}(\cdot):=\pi_\phi(\phi(\cdot))$ In the following section, $\phi$ is often selected by $\varepsilon$, and is treated as equivalent. Conventionally, notations with super/subscripts $\phi$ is also used to specify functions defined on the abstract system, and whenever we say $f_\phi\in\CAL{F}$ where $\CAL{F}$ is a function class defined on the true system, we mean $\exists f\in\CAL{F}, f(\phi(\cdot))=f_\phi(\phi(\cdot))$.

\paragraph{Other Notations:} We denote state-action occupancy as $d^{\pi}(s,a):=(1-\gamma)\sum_{k=1}^\infty\Pr_\pi(S_k=s,A_k=a)$. $J(\pi)$ represents the expected reward of a policy $\pi$, and $J_{\hat{Q}}(\pi)$ is the estimated reward of $\pi$ using approximation function $\hat{Q}$.

\section{Unified Analysis Overview}

In this section, we briefly explain how the geometry of the belief state space can help characterize the sample complexity of off-policy evaluation for POMDPs, and what our result looks like in general. We also present the basics of belief space, abstraction on the belief space induced by a $\varepsilon$-cover, and the assumptions related to the belief metric. %In this section we aim to build a meta theorem  that links the covering number with coverage (?)... \yplu{refine the sentence here, the first sentence of each section need to explain what you will do in this section}

\paragraph{Belief State Space and Smoothness Condition:}
Since one cannot observe the latent state directly, a prediction of the current state can be made using the information from the entire history of observations and actions. 
\begin{comment}
In chapter 4 and 5, we denote the history at time step $h$ to be
\begin{align}
    \tau_h=(o_1,a_1,o_2,a_2,\cdots,o_{h-1},a_{h-1},o_h)\in\mathcal{H},
\end{align}
\end{comment}
We denote the history at time step $h$ to be
$\tau_h=(o_1,a_1,o_2,a_2,\cdots,o_{h-1},a_{h-1})\in\mathcal{H}_h\subset \CAL{H}$ and $\tau_h^+:=(\tau_h,o_h)\in\CAL{H}_h^+\subset\CAL{H}^+$.
Consequently the belief state $\mathbf{b}(\tau_h^+):=\Pr(s_h|\tau_h^+)$ is an element of $\Delta(\mathcal{S})\subset\mathbb{R}^{|\mathcal{S}|}$ when $|\mathcal{S}|<\infty$. We use $\mathcal{B}$ to denote belief state space such that $\mathcal{B}=\{b:\exists h\in \mathbb{N}\ \exists \tau_h^+, \mathbf{b}(\tau_h^+)=b\}$. %\yplu{can you define the abstraction here mathmetically?}
\begin{comment}
\begin{assumption}\label{asmp:injection}
    The belief mapping $\mathbf{b}:\mathcal{H}\to\mathcal{B}$ is an injection (and thus a bijection).
\end{assumption}
%\yplu{under this assumption, why it's partial observed? it's almost fully observed right?}\zyh{addressed.}
    The assumption is especially natural when considering very large latent state space and therefore very high-dimensional belief state space.
    Consequently, $|\mathcal{B}|=\infty$ \yplu{maybe define $B$ depend on horizon but finally it's exp enlarging } unless we truncate an infinitely long tail from the history that we consider and bares a reasonable truncation error. In a following example and section \ref{finite_horizon}, we will discuss the case of $|\mathcal{B}|<\infty$, which gives us some worst-case properties.
\end{comment}
Consider a common case when such $\mathbf{b}$ is a bijection, then $\CAL{B}$ becomes a perfect proxy for $\CAL{H}$, of which the cardinality grows exponentially with the horizon. In infinite horizon cases, $|\CAL{B}|=\infty$, yet considering the compactness of a bounded subset of $\R^{|\CAL{S}|}$, cluster points of $\CAL{B}$ must exist.
 For simplicity, we assign distinct belief copies to histories that share the same belief state distribution, making the belief space metric a pseudo-metric. We denote the policy of interest $\tilde{\pi}(\tau_h^+)=\pi(\mathbf{b}(\tau_h^+)):\mathcal{H}^+\to\Delta(\mathcal{A})$, which is used to sample an action when given a history. Similarly for value function $\tilde{V}(\tau_h^+)=V(\mathbf{b}(\tau_h^+))$. Since $\tilde{V},\tilde{\pi},\tau_h^+\in\CAL{H}^+$ one-to-one correspond to $V,\pi,b\in\CAL{B}$, we slightly abuse our notation and treat them as equivalent for the rest of the passage, i.e. whenever we mention $b\in\CAL{B}$, we also mean the corresponding $\mathbf{b}^{-1}(b)\in\CAL{H}^+$, especially when describing the algorithms, since they only see the data trajectories instead of actual beliefs.%\zyh{added some explanation}

 %Footnote for the table
\footnotetext[1]{``Worst-case coverage'' refers to the worst case for the most exploratory data-collection distribution.}
\footnotetext[2]{$\mathrm{Covering(\mathcal{B},\varepsilon)}$ denotes the $L_1$ covering number of $\mathcal{B}$.}
\footnotetext[3]{For $H \to \infty$, we assume worst-case coverage grows as a subpolynomial power $\alpha_0 \leq 1$ (not logarithmic, which would trivially remove the curse of horizon). In the FDVF case, specific forgetting rates may be required.}

Then we introduce the core idea of using belief space metric to lower the complexity of the potentially exponential belief space, that is through covering. By introducing an $\varepsilon$-cover as a abstraction of the original belief space, we can treat near belief states as one, making the space simpler. This is formalized below with a similar idea as an \emph{$\varepsilon$-sufficient statistic} in~\cite{franccois2019overfitting, subramanian2022approximate}.

\paragraph{Abstraction Induced by Covering.}
    Consider the belief space $\CAL{B}$, for any $\varepsilon>0$ and a $\varepsilon$-cover $\CAL{C}_\varepsilon\subset\CAL{B}$ (Defined in Appendix~\ref{appendix:abstract}). There exists an abstraction $\phi:\CAL{B}\to\CAL{C}_\varepsilon$ such that $\forall b\in\CAL{B},\ \|\phi(b)-b\|_1\leq \varepsilon$. Select any such $\phi$, and a family of measure $\{p_x\}_{x\in\CAL{C}_\varepsilon}$ mentioned in Section~\ref{sect:prelim}, then an abstract belief MDP is defined, we refer to which as the abstract system.

To obtain a meaningful result, it is important for us to limit our attention to a subset of all possible policies, i.e. those that presents stability. This is characterized by the two core structural assumptions on the policy of interest, primarily introduced in Lipchitz-MDP literature~\cite{pirotta2015policy, gelada2019deepmdp}:

\begin{assumption}[Local Stability]\label{asmp:lip_policy}
   $\forall b_1,b_2\in\CAL{B},\ \|\pi(b_1)-\pi(b_2)\|_1\leq L_{\pi}\|b_1-b_2\|_1$. 
\end{assumption}
\begin{assumption}[Value Stability]\label{asmp:Lipchitz_value}
%\zyh{Check~\cite{shani2013survey}, $V(b)=\max_{\alpha \in \{\alpha_i\}_{i=1}^n}\langle b,\alpha\rangle$}
    $
         \sup_{\substack{b_1,b_2\in\CAL{B}\\\varepsilon\geq 0,\phi_\varepsilon}}  |V^{[\pi_{\phi_\varepsilon}]_{\rm true}}\allowbreak{}(b_1)-V^{[\pi_{\phi_\varepsilon}]_{\rm true}}(b_2)|/{\|b_1-b_2\|_1}\leq L_V<\infty.
    $
    %\zyh{There are close relationships between this and the former assumption, which is also covered in the proof of the next theorem. Can elaborate it separately later.}
\end{assumption}

\begin{remark}
    Assumption~\ref{asmp:lip_policy} is made by the intuition that a good belief state policy should treat two similar belief state similarly, and thus should itself have some local stability. Assumption~\ref{asmp:Lipchitz_value} measures the stability of a policy's long-term return. As indicated by the following Theorem~\ref{Thm:lip_optimal_value}, it can also be viewed as a proxy for how closely a policy resembles the optimal policy.
\end{remark}
\begin{theorem}[Lemma 1 in~\cite{lee2007makes}]\label{Thm:lip_optimal_value}
    For any $b_1,b_2\in\CAL{B}$, $|V^*(b_1)-V^*(b_2)|\leq\frac{R_{\rm max}}{1-\gamma}\|b_1-b_2\|_1$.
\end{theorem}
This shows that the optimal value function is $\frac{R_{\rm max}}{1-\gamma}$-value stable. Apart from the inherent stability of optimal value, people have studied this stability property not just in POMDPs, but also in continuous state MDPs. This line of work, e.g.~\cite{pirotta2015policy,asadi2018lipschitz,gelada2019deepmdp}, were explored in various contexts, but in such cases, the stability in value weren't as natural as in POMDPs, since unlike in belief spaces, the system dynamic in a continuous state MDP may not be smooth w.r.t. its intrinsic metric.

In general, the two assumption holds with some finite constant $L_\pi$ and $L_V$, but the worst case scaling of them could be exponentially large. However, policies with malignant stability are often bad and uninteresting, and it is efficient for us to only study the behavior of those good policies. With that said, our analysis applies to both cases, and the smaller the stability constants are, the more tractable our bound becomes. Either way, our bound will be no worse than the original.

\subsection{Unified Analysis In a Nutshell}

\begin{comment}
\yplu{have a literature review paragraph/a table to show how strong are the litearture assumption and what's the relationship between your work with others. The most important thing is the relatioship between your work with others.}
By assuming the Lipchitz continuity of value function w.r.t. belief state, we arrived at a PAC sample complexity guarantee using belief space coverage assumption, the general proof follows the process as demonstrated below: 
\end{comment}
\begin{figure}[h]
\centering
\begin{tikzpicture}[node distance=3.0cm, thick, >=Stealth, scale=0.65, every node/.style={transform shape}]

    % Nodes with rounded corners
    \node (true_system) [rectangle, rounded corners=10pt, draw, minimum width=3.0cm, minimum height=1.2cm, align=center] {True System};
    \node (binning_system) [rectangle, rounded corners=10pt, draw, minimum width=3.0cm, minimum height=1.2cm, align=center, below of=true_system] {Abstract System};
    \node (true_estimate) [rectangle, rounded corners=10pt, draw, minimum width=3.0cm, minimum height=1.2cm, align=center, right of=true_system, xshift=3cm] {Estimate from\\ True System};
    \node (binning_estimate) [rectangle, rounded corners=10pt, draw, minimum width=3.0cm, minimum height=1.2cm, align=center, right of=binning_system, xshift=3cm] {Estimate from\\ Abstract System};

    % Arrows and labels
    \draw[->, red, thick] (true_system) -- node[above, black] {\shortstack{True System\\ Coverage \textbf{\color{red}\ding{55}}}} (true_estimate);
    \draw[->, darkgreen, thick] (true_system) -- node[left, black] {\shortstack{\ding{192} State\\ Abstraction}} (binning_system);
    \draw[->, darkgreen, thick] (binning_system) -- node[below, black] {\shortstack{\ding{193} Abstract\\ System Coverage}} (binning_estimate);
    \draw[->, darkgreen, thick] (binning_estimate) -- node[right, black] {\shortstack{\ding{194} Gap of True\\ and Abstract\\ Algorithm}} (true_estimate);

    % Policy labels
    \node at ([yshift=0.3cm]true_system.north) {policy: $\pi$};
    \node at ([yshift=-0.3cm]binning_system.south) {policy: $\pi_\Phi$};
    \node at ([yshift=0.3cm]true_estimate.north) {policy: $\pi$};
    \node at ([yshift=-0.3cm]binning_estimate.south) {policy: $\pi_\Phi$};

\end{tikzpicture}
\caption{Pipeline of the analysis %\yplu{you should have an exmaple showing you don't have ture system coverage but have the binning system coverage. better to have a lemma to show true system coverage leads to binning system coverage}
}
\label{fig:system_abstraction}
\end{figure}
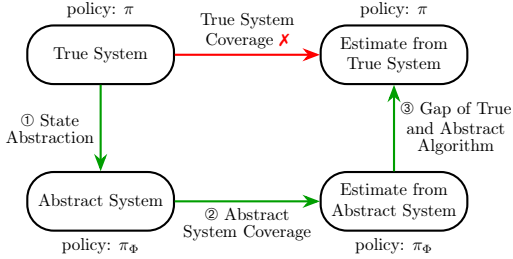

Specifically as shown in Figure~\ref{fig:system_abstraction}, in step 1, we descend the true belief space MDP system (resp. policy $\pi$) to an abstract system (resp. abstract policy $\pi_{\phi}$). Using similar ideas of state abstraction, we control the abstraction error using the size of bins $\varepsilon$. %We also show that belief space bisimulation assumption can be replaced by Lipchitz value function using chaining techniques.\yplu{this is your contribution?} 
In step 2, we execute the algorithm on the abstract system, with the coverage assumption for the abstract belief space, which can be much more tractable than the coverage of the true system due to the curse of horizon. We also provide Theorems~\ref{Thm:coverage_compare_2}, and~\ref{Thm:coverage_compare_3} to show that abstract coverage is no worse than the original coverage. %\yplu{the coverage assumption in the worst case has the curse of horizon right? you can say in the worst case they matches, but in average case this can improve.} 
Eventually for step 3, we utilize the stability property of value function again to control the difference between the real and the virtually executed algorithm on the same offline data. Combining all the analysis above, we obtain an estimation error bound without incorporating the traditional coverage assumption.

In this paper, we construct the abstraction using a $\varepsilon$-cover $\mathcal{C}_{\varepsilon}$ of the belief space, with definition in Appendix~\ref{appendix:abstract}. We state the following helpful lemma for controlling abstraction error.
\begin{lemma}\label{lemma:update_belief_lipchitz}
    $\forall a\in\CAL{A}, b_1,b_2\in\CAL{B}$, $\mathbb{E}_{o\sim P(\cdot|b_1,a)}[\|b_1^{o,a}-b_2^{o,a}\|_1]\leq 2 \|b_1-b_2\|_1$. Here $b^{o,a}$ denotes the updated next belief of $b$ after taking action $a$ and observing $o$.
\end{lemma}\begin{remark}
    Intuitively, after a pair of same action and observation $(a,o)$, the chances of two belief states sharing the same state becomes larger, resulting in the two next belief states become closer, i.e. a data processing inequality (DPI) should hold. However, such contraction property generally does not hold point wise as indicated in example~\ref{example:counter-exmp}, which also show that the Lipchitz value can go to infinity. The reason for that is that the belief update is a posterior instead of a Markov kernel, and a DPI only applies to the latter. However, the lemma shows that after taking expectation, the value is bounded by $2$. The proof can be found in Appendix~\ref{appendix:unified}.% \zyh{prove a upper bound.}
\end{remark}
\begin{proposition}\label{prop:one_hot_L_b}
    If for any $\tau_h\in\CAL{H}$, $\mathbf{b}(\tau_h)$ is one-hot, then $2$ in Lemma~\ref{lemma:update_belief_lipchitz} can be replaced with $1$.
\end{proposition}

\begin{theorem}\label{Thm:L_phi_1}  Under Assumption~\ref{asmp:lip_policy} and~\ref{asmp:Lipchitz_value}, for an abstraction $\phi_\varepsilon$  %\yplu{$\phi_\epsilon$} specified \yplu{depend on? ask chatgpt how to define the notation} 
depending on
$\varepsilon$, we have $\|V_{\rm true}^{\pi}-[V_{\rm bin}^{\pi_\phi}]_{\rm true}\|_\infty\leq L_\phi^{[1]}\varepsilon$, where $L_\phi^{[1]}:=\frac{(L_\pi+1)R_{\rm max}+2L_V}{1-\gamma}+\frac{\gamma R_{\rm max}L_\pi+R_{\rm max}}{(1-\gamma)^2}$, and $\CAL{V}$ is the function class for function estimation. See Appendix~\ref{appendix:abstract} for the proof.
\end{theorem}
%\yplu{put "For finite horizon POMDP, simply replace $1-\gamma$ with $H$." as a remark}
\begin{remark}
    For finite horizon POMDP, simply replace $(1-\gamma)^{-1}$ with $H$.
\end{remark}

Therefore, our previous assumptions enable a principled reduction from the exponentially large belief MDP to an abstract belief MDP, with a tractable approximation error. The abstract state space has cardinality on the order of the covering number, potentially mitigating the curse of horizon, as illustrated in Examples~\ref{example:covering_num_H_infty} and~\ref{example:Fast_forgetting_H_infty}. Formally speaking, we have the following meta-theorem, with the proof in Appendix~\ref{appendix:unified}.

%Now the start of new meta theorem
\begin{theorem}[Meta-theorem]\label{Thm:meta}
For a POMDP $\CAL{M}$, a policy $\pi$ and an OPE algorithm $\mathrm{Alg}:=\{\mathfrak{alg},\mathfrak{est}\}$ where $\mathfrak{alg}:\CAL{D}\to\CAL{V}$ learns a function $\hat{Q}^\pi\in \CAL{V}$ using offline dataset $\CAL{D}$ of size $n$, and $\mathfrak{est}:\CAL{V}\to\mathbb{R}$ estimates the expected reward of $\pi$ from $\hat{Q}^\pi$. We omit $\mathrm{Alg}$'s dependency on $\pi$ and $\mathfrak{est}$'s dependency on $\CAL{M}$ in the notation.

Then for any $\varepsilon\geq 0$ and an abstraction $\phi:\CAL{B}\to\CAL{B}$ such that $\forall\;b_1,b_2\in\CAL{B}$, $\phi(b_1)=\phi(b_2)\Rightarrow\|b_1-b_2\|_1\leq \varepsilon$, we denote the algorithm executed on the abstract system as $\mathrm{Alg}^\phi:=\{\mathfrak{alg}^\phi,\mathfrak{est}^\phi\}$. If assumption~\ref{asmp:lip_policy},~\ref{asmp:Lipchitz_value} holds, and that there exists an $L_\phi^{[2]}$ such that $|\mathfrak{est}(Q)-\mathfrak{est}^\phi(Q)|\leq L_\phi^{[2]}\varepsilon$ for all $Q\in\CAL{V}$, we also consider when $\mathrm{Alg}$ admits an finite sample estimation error on the abstract system of the form $|\mathfrak{est}^\phi(\hat{Q}^\pi)-\mathfrak{est}^\phi(Q_\phi^\pi)|\leq C_\pi^\phi\cdot\sqrt{\|\CAL{V}\|_\infty\cdot\bigg(\frac{1}{n}\log\frac{|\CAL{V}|}{\delta}\bigg)^\alpha+L_{\CAL{E}}\varepsilon},\;\; w.p.>1-\delta$,
% \begin{align*}
%     &|\mathfrak{est}^\phi(\hat{Q}^\pi)-\mathfrak{est}^\phi(Q_\phi^\pi)|\leq C_\pi^\phi\cdot\\ &\qquad\quad\sqrt{\|\CAL{V}\|_\infty\cdot\bigg(\frac{1}{n}\log\frac{|\CAL{V}|}{\delta}\bigg)^\alpha+L_{\CAL{E}}\varepsilon}\quad w.p.>1-\delta
% \end{align*}
where $C_\pi^\phi$ is a constant, $\|\CAL{V}\|_\infty,\; |V|$ respectively denotes the boundedness and cardinality of the function class for function approximation. Then we have $|\mathfrak{est}(\hat{Q}^\pi)-\mathfrak{est}(Q^\pi)|\leq L_\phi+C_\pi^\phi\cdot\sqrt{\|\CAL{V}\|_\infty\cdot\bigg(\frac{1}{n}\log\frac{|\CAL{V}|}{\delta}\bigg)^\alpha+L_{\CAL{E}}\varepsilon},\;\; w.p.>1-\delta$.
% \begin{align*}
%     &|\mathfrak{est}(\hat{Q}^\pi)-\mathfrak{est}(Q^\pi)|\leq L_\phi+C_\pi^\phi\cdot\\ &\qquad\quad\sqrt{\|\CAL{V}\|_\infty\cdot\bigg(\frac{1}{n}\log\frac{|\CAL{V}|}{\delta}\bigg)^\alpha+L_{\CAL{E}}\varepsilon}\quad w.p.>1-\delta.
% \end{align*}
Here, $L_\phi:=L_\phi^{[1]}+L_\phi^{[2]}$ with $L_\phi^{[1]}$ defined in Theorem~\ref{Thm:L_phi_1}, $Q^\pi$ and $Q_\phi^\pi$ represent the ground truth estimators.
\end{theorem}

%Based on \yplu{lipschitzness of value function},  we aim to build the bound as follows. explain your bound ,discritization error can be controled by lipschitzness, the should coverage just on covering numnber, thus the coverage is small and have the potential to break curse of memory and curse of horizon. cite the examples  \yplu{refine the sentence here, remember use ChatGPT}.

\subsection{Why Coverage on Covering is Better?}
%\yplu{just use our is too vague}
In the following part, we showcase the general idea why our coverage is no worse than the original coverage by providing the two theorems as a complement to our Table~\ref{table:comparison}. Since directly comparing the occupancy of $\pi_e$ and the abstract occupancy of $\pi_e^\phi$ is difficult, so we turn to comparing the occupancy of $[\pi_e^\phi]_{\rm true}:=\tau_h\mapsto\pi_e^\phi(\phi(\tau_h))$, which generally have the same scaling as that of $\pi_e$. Proving the theorems (see Appendix~\ref{appendix:why_better}) uses an information-theoretic idea that the divergence between two probability measures becomes smaller on a coarser $\sigma$-algebra, using the variational representation of $f$-divergences.
%Original theorem for L_2\L_\infty belief coverage
\begin{comment}
\begin{theorem}\label{Thm:coverage_compare_2}
    Consider the $L_2$ belief coverage in the one-hot scenario. Then for any behavior policy $\pi_b$ and truncation abstraction $\phi_T$, there exists a $d^D_\phi\in \Delta(\CAL{S}^\phi\times \CAL{H}_T)$, such that for any $\pi_e$, we have
    \begin{align}
        &\qquad\qquad\mathbb{E}_{d^D_\phi}\bigg[\bigg(\frac{d^{\pi_e^\phi}_\phi(s_h,{\tau_{[h-T+1:h]}})}{d^D_\phi(s_h,{\tau_{[h-T+1:h]}})}\bigg)^2\bigg]\nonumber\\&\leq\mathbb{E}_{\pi_b}\bigg[\bigg(\frac{d^{[\pi_e^\phi]_{\rm true}}(s_h,{\tau_h})}{d^{\pi_b}(s_h,{\tau_h})}\bigg)^2\bigg]\sim \mathbb{E}_{\pi_b}\bigg[\bigg(\frac{d^{\pi_e}(s_h,{\tau_h})}{d^{\pi_b}(s_h,{\tau_h})}\bigg)^2\bigg]
    \end{align}
\end{theorem}

\begin{theorem}\label{Thm:coverage_compare_3}
    Consider the $L_\infty$ belief coverage in the one-hot scenario. Then for any behavior policy $\pi_b$ and truncation abstraction $\phi_T$, there exists a $d^D_\phi\in \Delta(\CAL{S}^\phi\times \CAL{H}_T)$, such that for any $\pi_e$, we have
    \begin{align}
        &\qquad\qquad\bigg\|\frac{d^{\pi_e^\phi}_\phi(s_h,{\tau_{[h-T+1:h]}})}{d^{D}_\phi(s_h,{\tau_{[h-T+1:h]}})}\bigg\|_\infty\nonumber\\&\leq \bigg\|\frac{d^{[\pi_e^\phi]_{\rm true}}(s_h,{\tau_h})}{d^{\pi_b}(s_h,{\tau_h})}\bigg\|_\infty\sim \bigg\|\frac{d^{\pi_e}(s_h,{\tau_h})}{d^{\pi_b}(s_h,{\tau_h})}\bigg\|_\infty
    \end{align}
\end{theorem}
\end{comment}

\begin{theorem}\label{Thm:coverage_compare_2}
    Consider the $L_2$ belief coverage in the one-hot scenario. Then for any behavior policy $\pi_b$ and truncation abstraction $\phi_T$, there exists a $d^D_\phi\in \Delta(\CAL{S}^\phi\times \CAL{H}_T)$, such that for any $\pi_e$, along with its abstract policy $\pi_e^\phi$ and the corresponding lifted version $[\pi_e^\phi]_{\rm true}$, we have $\mathbb{E}_{d^D_\phi}\bigg[\bigg(\frac{d^{\pi_e^\phi}_\phi(s_h,{\tau_{[h-T+1:h]}})}{d^D_\phi(s_h,{\tau_{[h-T+1:h]}})}\bigg)^2\bigg]\leq\mathbb{E}_{\pi_b}\bigg[\bigg(\frac{d^{[\pi_e^\phi]_{\rm true}}(s_h,{\tau_h})}{d^{\pi_b}(s_h,{\tau_h})}\bigg)^2\bigg]$.
\end{theorem}

\begin{theorem}\label{Thm:coverage_compare_3}
    Same result for the $L_\infty$ belief coverage that $\bigg\|\frac{d^{\pi_e^\phi}_\phi(s_h,{\tau_{[h-T+1:h]}})}{d^{D}_\phi(s_h,{\tau_{[h-T+1:h]}})}\bigg\|_\infty\leq \bigg\|\frac{d^{[\pi_e^\phi]_{\rm true}}(s_h,{\tau_h})}{d^{\pi_b}(s_h,{\tau_h})}\bigg\|_\infty$.
\end{theorem}

Next, we provide illustrative examples to show the superiority our result under certain structures.%, our result can resolve the curse of horizon/memory.
\begin{example}\label{example:covering_num_H_infty}
    Consider a belief space with smoothness structure [Detailed Definition in Appendix~\ref{appendix:why_better}]. With coverage sublinear polynomial to the worst case, we have a finite sample guarantee of $O(\frac{(C|\CAL{S}|L_\CAL{E}mR_{\rm max}^2)^{\frac{1}{4}}}{(1-\gamma)^{\frac{3}{2}}}\cdot (\frac{1}{n}\log\frac{|\mathcal{F}|}{\delta})^{\frac{1}{8}})$, where $C$, $m$ are constants related to the smoothness property.
\end{example} %\yplu{in appendix create a exmpale, write every detail in the appendix. here [Detialed Definition Example xxx in Appendix ] cite the \cite{lee2007makes} in the appendix. A paper need to be self-contained. You would never imagine a reviewer/reader to open another paper while reading your paper.}

\begin{example}\label{example:Fast_forgetting_H_infty}
    Consider a fast forgetting policy with forgetting speed $T(\varepsilon)=O(\log\frac{1}{\varepsilon})$, then with coverage sublinear polynomial to the worst case, we can obtain a finite sample guarantee of $O(\frac{\max\{\|\CAL{V}\|_\infty,\|\Theta\|_\infty\}}{(1-\gamma)^2}\cdot (\frac{1}{n}\log\frac{|\CAL{V}||{\Theta}|}{\delta})^{\frac{1}{4}})$. If we make a even stronger assumption than logarithmical scaling memory, i.e. strictly short-term memory, then the result goes back to what's discussed in~\cite{NEURIPS2023_3380e811, zhang2024cursesfuturehistoryfuturedependent}.
\end{example}
%\yplu{same thing here.}

%\zyh{The following theorems prove that our coverage is no worse than the original coverage.}

%\begin{remark}
%    Despite $[\pi_e^\phi]_{\rm true}$ in this case is a short-term memory policy, we compare our result with $\frac{d^{[\pi_e^\phi]_{\rm true}}(s_h,{\tau_h})}{d^{\pi_b}(s_h,{\tau_h})}$ instead of $\frac{d^{[\pi_e^\phi]_{\rm true}}(s_h,{\tau_{[h-T:h]}})}{d^{\pi_b}(s_h,{\tau_{[h-T:h]}})}$ because we are generally interested in $\frac{d^{\pi_e}(s_h,{\tau_h})}{d^{\pi_b}(s_h,{\tau_h})}$, where $\pi_e$ isn't short-term memory. However, directly comparing the occupancy of $\pi_e$ and the abstract occupancy of $\pi_e^\phi$ is difficult, so we turn to comparing the occupancy of $[\pi_e^\phi]_{\rm true}$, which generally have the same scaling as that of $\pi_e$.
%\end{remark}

\section{Examples of Application}
In this section, we apply our unified analysis on two different types of OPE algorithms, namely, the double sampling Bellman error minimization algorithm and future-dependent value function, aiming for a more sample efficient guarantee.
%meta theorem (Theorem \ref{Thm:meta}) to ... \yplu{refine the sentence here, the first sentence of each section need to explain what you will do in this section and how your bound is better/break the cures of horizion}
\subsection{Analysis on Bellman Error Minimization Algorithms}\label{Chapter:double_sampling}
    %\yplu{have a paragraph to introduce the algorithm, same to follow the language here.}

    %\paragraph{Algorithm on Belief Space MDP}%\zyh{Very standard textbook analysis here, because our method is to "boost" the original guarantees by executing the guarantee on a less complex space, therefor the analysis is not our contribution. But the abstract coverage Assumption~\ref{asm:belief_policy_coverage} is important in our analysis.}
    \paragraph{Double Sampling.} Consider a Bellman error minimization algorithm using double sampling, each offline data contains two tuple $(b, a, r, b_A')$ and $(b, a, r, b_B')$ with the latter sampled independently after the system resets to belief $b$. The corresponding estimator can be written as
    $\hat{Q}^{\pi}=\argmin_{f\in\mathcal{F}}\mathcal{E}(f,\pi)$ where $\mathcal{E}(f,\pi)=\mathbb{E}_{\mathcal{D}}[(f(b,a)-(r+\gamma f(b'_A,\pi)))(f(b,a)-(r+\gamma f(b'_B,\pi)))]$
    
\begin{comment}
    Consider a Bellman error minimization algorithm using double sampling, which is executed both in the real world (simulator) and virtually in the abstract system, using the same piece of offline data. The optimization target can be written respectively as
    \begin{align}
        \hat{Q}^{\pi}=\argmin_{f\in\mathcal{F}}\mathcal{E}(f,\pi),\quad\hat{Q}_{\phi}^{\pi}=\argmin_{f\in\mathcal{F}}\mathcal{E}_\phi(f,\pi)
    \end{align}
    where\zyh{move to appendix, change to define covering}
    \begin{align}
        \mathcal{E}(f,\pi)=&\ \mathbb{E}_{\mathcal{D}}[(f(b,a)-(r+\gamma f(b'_A,\pi)))(f(b,a)-(r+\gamma f(b'_B,\pi)))]\\
        \mathcal{E}_\phi(f,\pi)=&\ \mathbb{E}_{\mathcal{D}}[(f(\phi(b),a)-(r_\phi+\gamma f(\phi(b'_A),\pi_\phi)))(f(\phi(b),a)-(r_\phi+\gamma f(\phi(b'_B),\pi_\phi)))].
    \end{align}
\end{comment}
    
    Instead of assuming standard coverage on the true system, we adopt the following abstract covering assumption on the abstract system.
    \begin{assumption}[Abstract Policy Coverage]\label{asm:belief_policy_coverage}
        $\|d^{\pi_\phi}/d^{{D}}\|_\infty\leq C_\pi(\phi)<\infty$
    \end{assumption}
    %Note that the coverage value is dependent of the abstraction mapping $\phi$. The the best behaving data collection distribution $d^\mathcal{D}$ has a worst case coverage would scale as $|\mathcal{C}_\varepsilon|$, which indicates the covering number for $\varepsilon$.
    \begin{remark}\label{remark:discuss_coverage_dbsp}
It is worth noting that the coverage $C_\pi(\phi)$ here depends on the specific abstraction mapping $\phi$. Under the most exploratory data collection distribution $d^{D}$, the worst-case growth rate of $C_\pi(\phi)$ is approximately aligned with $|\mathcal{C}_\varepsilon|$, which denotes the $\varepsilon$-covering number.
The benefit of the belief-policy coverage Assumption~\ref{asm:belief_policy_coverage} lies in its potential to outperform coverage assumptions in the original space. Using an abstract belief space allows the exponentially large history space to be reduced to a space with size of $\varepsilon$-covering number.
\end{remark}

And also a standard realizability assumption.
\begin{assumption}[Abstract Realizability]\label{asmp:realizability}
$Q^{\pi_\phi}_\phi \in \CAL{F}$, which according to our notation, is short for $\exists f\in\CAL{F}, f(\phi(\cdot))=Q^{\pi_\phi}_\phi(\phi(\cdot))$ since $Q^{\pi_\phi}_\phi$ is defined on the abstract system.
\end{assumption}

Noticed that we previously assumed the stability of value function, whose equivalence to the Lipchitz continuity of $Q$-function at action $a$ can be easily proven. We now assume the function class $\mathcal{F}$ we use to approximate $Q$-function is also Lipchitz with regard to belief state.
\begin{assumption}[Lipchitz function class]\label{asm:lipchitz_function_class}
 $\forall f\in\mathcal{F}$, $\forall a\in\mathcal{A}$, $|f(b_1,a)-f(b_2,a)|\leq L_Q\|b_1-b_2\|_1$.
\end{assumption}
%\begin{remark}
%Note that the previous discussion assumed the function class satisfies the realizability Assumption~\ref{asmp:realizability}. Therefore, a necessary condition for Assumption~\ref{asm:lipchitz_function_class} to coexist with it is that $L_Q \geq L_V$.
%\end{remark}
Then, we can provide the value of $L_\phi^{[2]}$ defined in Theorem~\ref{Thm:meta} for this special case, and furthermore, the eventual guarantee for double sampling algorithm using the aforementioned assumptions and methods of analysis. Proofs in Appendix~\ref{appendix:dbsp}.
\begin{theorem}\label{Thm:db_sp_Lphi}
    If Assumption~\ref{asm:lipchitz_function_class} holds, then for $L_\phi^{[2]}$ defined in Theorem~\ref{Thm:meta}, $L_\phi^{[2]}=\frac{R_{\rm max}}{1-\gamma}+L_Q$.% And $L_\phi=L_\phi^{[1]}+L_\phi^{[2]}$.
\end{theorem}

%\subsection{Final Boosted Error Bound}

%Finally, by combining the above theorems, we arrive at the following conclusion:
%As a direct corollary of the meta Theorem~\ref{Thm:meta}, we have the following theorem.
\begin{theorem}\label{thm:final}
If Assumptions~\ref{asmp:realizability},~\ref{asm:lipchitz_function_class},%\ref{asmp:injection},
~\ref{asmp:lip_policy}, and~\ref{asmp:Lipchitz_value} all hold, then we have:
\begin{align*}
&\textstyle |J_{\hat{Q}^\pi}(\pi)-J(\pi)|\leq\inf_{\substack{\varepsilon\geq 0\\D(\varepsilon)}}\bigg(\frac{\sqrt{C_\pi(\varepsilon)}}{1-\gamma}\cdot\\&\qquad\qquad\textstyle
        \sqrt{\sqrt{\frac{32R_{\rm max}^4}{n(1-\gamma)^4}\cdot\log\frac{2|\mathcal{F}|}{\delta}}+L_{\CAL{E}}\varepsilon}+L_\phi\varepsilon\bigg)
\end{align*}
where $L_\CAL{E}=\frac{8R_{\rm max}}{1-\gamma}\cdot\big((1+\gamma)L_Q+\frac{R_{\rm max}}{1-\gamma}\big)$, $L_\phi$ is defined as in Theorem~\ref{Thm:meta} and $D(\varepsilon)$ stands for such $\varepsilon$ that satisfies realizability (Assumption~\ref{asmp:realizability}).
\end{theorem}

%This completes the full theoretical analysis for the double sampling algorithm. A similar analysis can be performed for the doubly robust and marginal importance sampling methods. This section has already laid the groundwork for that—indeed, once their finite sample guarantees corresponding to ~\eqref{Thm:finite_sample_guarantee_double_sampling} are given, the remaining steps in the analysis are nearly identical.

%And we finish the total analysis. Notice that we will need to assume $\sqrt{C_\pi(\varepsilon)\cdot\varepsilon}$ would tend to zero when $\varepsilon\to 0$. In finite horizon setting (Section \ref{finite_horizon}), this is automatically satisfied, but in infinite setting, we will need to assume that covering number has a increasing rate slower than $1/\varepsilon$ when $\varepsilon\ll 1$. Generally, covering number of belief space characterizes the hardness of OPE, and in infinite horizon cases, the analysis would depend on the rate of the covering number and even may not be able to control. But of course, this is assuming the most exploring data sampling distribution , and could be different when the exploring policy is good enough (i.e. can fit the target occupancy better.)
%\zyh{Can change Lipchitz assumption of policy and value function...but the affect would be limited.}

%A useful corollary that presents a boosted finite sample guarantee is as followed:
\begin{corollary}[Finite sample guarantee]\label{coro:finite_sample_1}
    If Assumptions~\ref{asm:lipchitz_function_class},\ref{asmp:lip_policy}, and~\ref{asmp:Lipchitz_value} all hold, then for all $n$ satisfying

   % $\big(\frac{32R_{\rm max}^4}{n(1-\gamma)^4}\cdot\log\frac{2|\mathcal{F}|}{\delta}\big)^{\frac{1}{4}}\leq \frac{L_{\CAL{E}}}{L_\phi}\cdot \frac{\sqrt{2C_\pi^n}}{1-\gamma}$,
$n\geq 8R_{\rm max}^4(L_\phi/L_\CAL{E})^4\log(2|\CAL{F}|/\delta)$,
   and the abstraction $\phi$ induced by $\varepsilon$-cover with $\varepsilon=\frac{1}{L_\CAL{E}}\sqrt{\frac{32R_{\rm max}^4}{n(1-\gamma)^4}\cdot\log\frac{2|\mathcal{F}|}{\delta}}$ %\yplu{this is not clear} 
satisfies Assumption~\ref{asmp:realizability}, we have
$
|J_{\hat{Q}^\pi}(\pi)-J(\pi)|\leq\frac{2\sqrt{C_\pi^n}}{1-\gamma}\cdot
        \big(\frac{128R_{\rm max}^4}{n(1-\gamma)^4}\cdot\log\frac{2|\mathcal{F}|}{\delta}\big)^{\frac{1}{4}},
$
where $C_\pi^n:=C_\pi\big(\frac{1}{L_\CAL{E}}\sqrt{\frac{32R_{\rm max}^4}{n(1-\gamma)^4}\cdot\log\frac{2|\mathcal{F}|}{\delta}}\big)$.
\end{corollary}
%\yplu{have a remark here briefly discuss why your bound is better/break the cures of horizion}
\begin{remark}
    The guarantee obtained using our method relies on the coverage defined on the abstract system, which is more tractable than the original coverage as discussed in Remark~\ref{remark:discuss_coverage_dbsp} and Table~\ref{table:comparison}. Moreover, with appropriate belief space smoothness condition (Example~\ref{example:covering_num_H_infty}), our result provides a polynomial finite sample guarantee while the original bound goes to infinity. 
\end{remark}

\subsection{Future-Dependent Value Function.}\label{Chapter:FDVF}
FDVF was proposed targeting memoryless policies.  %In short, the algorithm requires modifications to work under memory-based policies and suffers from the curse of memory in the absence of any assumptions, which is the problem we are trying to solve. 
Here we introduce the memory-based version of FDVF, which suffers from the "curse of memory" as discussed in~\cite{zhang2024cursesfuturehistoryfuturedependent}.
We first introduce the respective definition of future space $\CAL{F}'$ as $f'_h:=(o_h,a_h,o_{h+1},a_{h+1},\cdots,o_H,a_H)\in\CAL{F}'_h\subset\CAL{F}'.$

From this point forward, for convenience, we will write $(f'_h,\tau_h)$ simply as $f_h$. Similarly, we will treat $\CAL{F}'$ as the original future space, and define $\CAL{F}:=\CAL{F}'\times\CAL{H}$ as the new space of “(future-history) pairs.” This is because $\tau_h$ can be considered a part of the extended future, or equivalently, the future is duplicated separately for each history sequence. The future-dependent value function $V_{\CAL{F}}$ is any such function that satisfies $\mathbb{E}_{\pi_b}[V_\CAL{F}(f_h,\tau_h)|s_h,\tau_h]=V_{\CAL{S}}^{\pi_e}(s_h,\tau_h) $ with the RHS being the value function of $\pi_e$, and is a zero point of the following two
Bellman Residual Operators.

\begin{definition}[Memory-Based Bellman Residual Operator]%\zyh{need a definition for FDVF.}
We define $(\CAL{B}^{(\CAL{S},\CAL{H}_T)}V)(s_h,\tau_{\scriptscriptstyle[h-T+1:h]}):=\mathbb{E}_{\substack{ a_{ 1:h}\sim\pi_e\\ a_{ h+1:H}\sim\pi_b}}[r_h+V(f_{h+1})|s_h,\tau_{\scriptscriptstyle[h-T+1:h]}]-\mathbb{E}_{\substack{ a_{ 1:H-1}\sim\pi_e\\ a_{ h:H}\sim\pi_b}}[\allowbreak{} V(f_h)|s_h,\tau_{\scriptscriptstyle[h-T+1:h]}]$, and $(\CAL{B^H}V)(\tau_h):= \mathbb{E}_{\substack{a_{1:h}\sim\pi_e\\a_{h+1:H}\sim\pi_b}}[r_h\allowbreak{}+ V(f_{h+1})|\tau_h]-\mathbb{E}_{\substack{a_{1:h-1}\sim\pi_e\\a_{h:H}\sim\pi_b}}[V(f_h)|\tau_h]$.
\begin{comment}
\begin{align}
&\quad(\CAL{B}^{(\CAL{S},\CAL{H}_T)}V)(s_h,\tau_{\scriptscriptstyle[h-T+1:h]})\nonumber\\&:=\mathbb{E}_{\substack{ a_{ 1:h}\sim\pi_e\\ a_{ h+1:H}\sim\pi_b}}[r_h+V(f_{h+1})|s_h,\tau_{\scriptscriptstyle[h-T+1:h]}]\nonumber\\&\qquad\qquad\qquad-\mathbb{E}_{\substack{ a_{ 1:H-1}\sim\pi_e\\ a_{ h:H}\sim\pi_b}}[V(f_h)|s_h,\tau_{\scriptscriptstyle[h-T+1:h]}]\label{def:5-2}
\end{align}
\begin{align}
&\quad(\CAL{B^H}V)(\tau_h):= \mathbb{E}_{\substack{a_{1:h}\sim\pi_e\\a_{h+1:H}\sim\pi_b}}[r_h+V(f_{h+1})|\tau_h]\nonumber\\&\qquad\quad\qquad\qquad\qquad\qquad-\mathbb{E}_{\substack{a_{1:h-1}\sim\pi_e\\a_{h:H}\sim\pi_b}}[V(f_h)|\tau_h]\label{def:5-3}
\end{align}
\end{comment}
\end{definition}

\paragraph{Memory-Based Algorithm.}
For memory-based policies, we define $\mu(a_h,\tau_h^+):=\frac{\pi_e(a_h|\tau_h^+)}{\pi_b(a_h|\tau_h^+)}$, then the min-max algorithm is defined as follows:
\begingroup
\allowdisplaybreaks
\begin{align*}
\hat{V}_\CAL{F} &= \argmin_{V\in\CAL{V}} \max_{\theta \in \Theta}\sum_{h=1}^H\mathbb{E}_{\CAL{D}}[\{\mu(a_h,\tau_h^+)(r_h+V(f_{h+1}))\nonumber\\&\qquad\qquad\qquad\qquad-V(f_h)\}\theta(\tau_h)-\frac{1}{2}\theta(\tau_h)^2]
\end{align*}
\endgroup

\paragraph{FDVF Analysis Pipeline.}

%\zyh{everything from now on is our contribution. Except for Theorem~\ref{Thm:5-16} which is only a small change to the original result in FDVF literature.}
The analysis of FDVF follows a structured framework that uses the previously introduced methodology, of which an illustrative Figure~\ref{fig:FDVF_system_abstraction} can be found in Appendix~\ref{appendix:FDVF_figure}. All proofs can also be found in Appendix~\ref{appendix:FDVF_figure}.

\paragraph{Abstraction Induced by Truncation Mapping.}

The first step in the approach is to introduce an abstraction mapping $\tilde{\phi}:\CAL{H}^+\to \CAL{H}_{T}^+$, where $T$ is the time window, and $\CAL{H}_T^+:=\bigcup_{t=1}^T(\CAL{O}\times\CAL{A})^{t-1}\times\CAL{O}$ denotes the set of history sequences constrained by the window $T$. %The specific definition of $\tilde{\phi}$ is as follows:
\begin{align*}
\tilde{\phi}(o_1,a_1,\dots,o_h) :=
\begin{dcases}
(o_{\scriptscriptstyle h-T+1}, a_{\scriptscriptstyle h-T+1}, \dots, o_{\scriptscriptstyle h}), & h \geq T \\
\mathrm{id}, & h < T
\end{dcases}
\end{align*}

To reuse the previous analysis, we also introduce an abstraction mapping $\phi:\CAL{B}\to\CAL{B}$ that operates directly on belief states. %\zyh{For simplicity, we assign distinct belief copies to histories that share the same belief state distribution, making the belief space metric a pseudo-metric.} 
The mapping $\phi$ satisfies $\phi(\mathbf{b}(\tau_h^+))=\mathbf{b}(\tilde{\phi}(\tau_h^+))$. Since this mapping $\phi$ depends on the time window length $T$, we denote it as $\phi_T$. Notice that $\phi_T$ and $\tilde{\phi}_T$ are one-to-one, we treat them equivalently. Now we provide the fast-forgetting assumption of POMDP and the policy.

\begin{assumption}[Fast-Forgetting POMDP]\label{asmp:fast_forgetting_POMDP}
For the abstraction mapping $\phi_T$ defined above, the following holds: for all $\varepsilon > 0$, there exists $T \in \N^+$ such that for all $b_1, b_2 \in \CAL{B}$, if $\phi_T(b_1) = \phi_T(b_2)$, then $\|b_1 - b_2\|_1 \leq \varepsilon$. The values of $T$ satisfying this condition form a function of $\varepsilon$, denoted $T_1(\varepsilon)$.
\end{assumption}
\begin{comment}
With this abstraction mapping, we can induce an abstract belief MDP. Moreover, since the abstraction $\tilde{\phi}$ (or equivalently $\phi$, as they correspond one-to-one) acts as a bijection for trajectory histories $\tau_h$ of length $h \leq T$, it effectively selects a representative from each partition for these trajectories.

Furthermore, observe that in this belief MDP, if for any two belief states $b, b' \in \CAL{B}$ we have $\Pr(b'|b) > 0$, then the last $T-1$ components of the sequence $\mathbf{b}^{-1}(b')$ must exactly match the first $T-1$ components of the sequence $\mathbf{b}^{-1}(b)$.
\textbf{This ensures the existence of a POMDP with the same $(\CAL{O}, \CAL{A})$ structure whose belief MDP is isomorphic to the belief MDP induced by $\phi$.}
\end{comment}

\begin{assumption}[Fast-Forgetting Policy]\label{asmp:fast_forgetting_policy}
    For the abstraction mapping $\phi_T$, it holds that for all $\varepsilon > 0$, there exists a $T \in \N^+$, such that for all $\tau_h^{[1]+}, \tau_h^{[2]+} \in \CAL{H}^+$ and all $\pi \in {\pi_e, \pi_b}$, if $\tilde{\phi}_T(\tau_h^{[1]+}) = \tilde{\phi}_T(\tau_h^{[2]+})$, then $\|\pi(\tau_h^{[1]+}) - \pi(\tau_h^{[2]+})\|_1 \leq L_\pi \varepsilon$. We denote the dependency of $T$ on $\varepsilon$ as $T_1(\varepsilon)$.
\end{assumption}
\begin{lemma}[Stability implies Fast-Forgetting]\label{lem:weaker}
If Assumption~\ref{asmp:fast_forgetting_POMDP} and Assumption~\ref{asmp:lip_policy} hold, then Assumption~\ref{asmp:fast_forgetting_policy} holds aotomatically, with $T_1=T_0$.
\end{lemma}

\paragraph{Conditions: Controlling Differences between Real and Abstract Algorithm.} %\yplu{write you want to do here. Here the text is sloppy and hard to follow.}
Since our analysis is build on the requirement that the virtually executed algorithm and the actual algorithm bear little difference, we first propose some conditions to restrain $\varepsilon$ from being too large.
\begin{definition}
    We define $\|\CAL{V}\|_\infty:=\max_{V\in\CAL{V}}\|V\|_\infty$ (similar for $\Theta$), $C_{\CAL{V}}:=\max\{\|\CAL{V}\|_\infty+1,\|\Theta\|_\infty\}$, $C_\mu := \max_h \max_{a_h, \tau_h^+} \mu(a_h, \tau_h^+)$, and
$
L_{\CAL{E}}:=3\big(\frac{2H{(C_\mu+1)L_\pi\|\CAL{V}\|_\infty\|\Theta\|_\infty}}{\min_h\min_{a_h,\tau_h^+}\pi_b(a_h|\tau_h^+)}+HC_\mu\|\CAL{V}\|_\infty\|\Theta\|_\infty + \frac{3 H^2 \max\{C_\mu\|\CAL{V}\|_\infty\|\Theta\|_\infty, \frac{1}{2}\|\Theta\|_\infty^2\}}{\min\{\min_h\min_{o_h,\tau_h}P(o_h|\tau_h), \min_h\min_{a_h,\tau_h^+}\pi_b(a_h|\tau_h^+)/L_\pi\}}\big).
$
\end{definition}
%First, we consider these two conditions that limits $\varepsilon$ from being too large: 
%\yplu{needs to be assumption?}\zyh{I use it as "for any $\varepsilon$ satisfying condition...".}
\begin{condition}\label{cond:0}
    The $\varepsilon$ is small enough that $L_\pi\varepsilon/\min_h\min_{a_h,\tau_h^+}\pi_b(a_h|\tau_h^+)\leq \frac{1}{2}$.
\end{condition}
\begin{condition}\label{cond:1}
    The $\varepsilon$ is small enough that $\frac{H \varepsilon}{\min\{\min_h\min_{o_h,\tau_h}P(o_h|\tau_h), \min_h\min_{a_h,\tau_h^+}\pi_b(a_h|\tau_h^+)/L_\pi\}}\leq 1$
\end{condition}
Condition~\ref{cond:0} assumes non-zero entries for the behavior policy $\pi_b$, which is known and chosen by the learner. This assumption is also used in literature like~\cite{zhang2025statistical}, from which we adopt the same philosophy.
In Condition~\ref{cond:1}, the probability $P(o_h|\tau_h)$ being uniformly above zero is a non-trivial assumption, but we'll show later in a simpler pipeline that this condition can actually be discarded.

%Then in order to complete $L_\phi$ in Theorem~\ref{Thm:meta}, w
We then state the following theorem about $L_\phi^{[2]}$.
\begin{theorem}\label{Thm:def_Lphi_fdvf}
If Assumptions~\ref{asmp:fast_forgetting_POMDP}%,~\ref{assumption:asmp:update_belief_lipchitz}
%,~\ref{asmp:Lipchitz_value}
,~\ref{asmp:fast_forgetting_policy} and~\ref{asmp:fast_forgetting_function_class} hold, then for any $\varepsilon > 0$, with $T \geq \max\{T_0(\varepsilon),T_1(\varepsilon),T_2(\varepsilon)\}$, we have $L_\phi^{[2]}=\|\CAL{V}\|_\infty$, with $L_\phi^{[2]}$ defined in Theorem~\ref{Thm:meta}.
\end{theorem}

\paragraph{Theoretical Guarantee of FDVF.}

%After completing the previous two subsections, we have essentially finished the error control of the three green arrows in Figure~\ref{fig:FDVF_system_abstraction}, thereby completing the overall analysis. 
The following theorem showcases the guarantee for FDVF under our unified analysis, with the given condition that indicates our selection of $\varepsilon$ should generally have the same scaling as finite-sample error term.
\begin{condition}\label{cond:2}
    For some uniform constant $C$, for the given $\varepsilon,n,\delta$, $L_{\CAL{E}}\varepsilon\leq \frac{eCHC_{\CAL{V}}^2C_\mu}{2n}\cdot\log\frac{4|\CAL{V}||\Theta|}{\delta}$.
\end{condition}

\begin{theorem}[Theoretical Guarantee of FDVF]\label{Thm:finite_sample_FDVF}
Suppose the abstract realizability condition $V_\CAL{F}^\phi \in \CAL{V}$ and the Bellman completeness condition $\forall V \in \CAL{V}, \CAL{B^H}V \in \Theta$ ($\CAL{B^H}$ here refers to the operator on the abstract system) hold, and Assumptions%~\ref{assumption:asmp:update_belief_lipchitz},
~\ref{asmp:Lipchitz_value},~\ref{asmp:fast_forgetting_POMDP},~\ref{asmp:fast_forgetting_policy}, and~\ref{asmp:fast_forgetting_function_class} are satisfied. For any $\varepsilon > 0$ satisfying condition~\ref{cond:0},~\ref{cond:1},~\ref{cond:2}, define $T=\max\{T_0(\varepsilon),T_1(\varepsilon),T_2(\varepsilon)\}$.
Then, for some uniform constant $c$, with probability at least $1 - \delta$, we have: 
\begingroup
\allowdisplaybreaks
\begin{align*}
&|J(\pi_e)-\mathbb{E}_{\pi_b}[\hat{V}(f_1)]|\leq L_\phi\varepsilon+\sqrt{H} \cdot \nonumber\\&\textstyle\qquad\max_{h\in[H]}\sup_{V\in\CAL{V}}\sqrt{\frac{\mathbb{E}_{\pi_e^\phi}[(\CAL{B}^{(\CAL{S},\CAL{H}_T)}V)(s_h,\tau_{[h-T+1:h]})^2]}{\mathbb{E}_{\pi_b^\phi}[(\CAL{B^H}V)(\tau_h)^2]}}\cdot\nonumber\\&\textstyle\qquad\qquad\qquad\sqrt{\frac{cHC_{\CAL{V}}^2C_\mu}{n}\log\frac{|\CAL{V}||{\Theta}|}{\delta}+L_{\CAL{E}}\varepsilon} 
\end{align*}
\endgroup
\end{theorem}
%\yplu{if we just have 9 page, can we just have finite sample guarantee results?}

\begin{corollary}[Boosted finite sample guarantee]\label{coro:finite_sample_2}
    For $n$ large enough with necessary realizability and completeness condition, we have a finite sample guarantee:
    \begingroup
    \allowdisplaybreaks
    \begin{align*}
    &\textstyle|J(\pi_e)-\mathbb{E}_{\pi_b}[\hat{V}(f_1)]|\leq \sqrt{H} \cdot\sqrt{\frac{cHC_{\CAL{V}}^2C_\mu}{n}\log\frac{|\CAL{V}||{\Theta}|}{\delta}} \cdot\\&\qquad\textstyle\max_{h\in[H]}\sup_{V\in\CAL{V}}\sqrt{\frac{\mathbb{E}_{\pi_e^\phi}[(\CAL{B}^{(\CAL{S},\CAL{H}_T)}V)(s_h,\tau_{[h-T+1:h]})^2]}{\mathbb{E}_{\pi_b^\phi}[(\CAL{B^H}V)(\tau_h)^2]}}
    \end{align*}
    \endgroup
\end{corollary}
%\yplu{have a remark here briefly discuss why your bound is better/break the cures of horizon}
\paragraph{A Simpler Pipeline: Abstracting Only the Policy.}

Revisiting the above analysis and noticeably one step is actually unnecessary, namely, the abstraction from the original POMDP to the short-term memory POMDP. That's because the memory dependency of the policy is the real root of the “curse of memory.” Notably, the introduction of Assumption~\ref{asmp:Lipchitz_value} %,~\ref{assumption:asmp:update_belief_lipchitz} 
and~\ref{asmp:fast_forgetting_POMDP} are all for the sake of bounding the abstraction error of the POMDP itself, and therefore can be eliminated for FDVF. \textbf{This shows a significant advantage of FDVF comparing to history-as-state MDP that the "curse of memory" is much easier to handle than "the curse of horizon", since for the latter, abstracting the POMDP itself is inevitable.} When we only abstract the policy, the previous condition~\ref{cond:0} and~\ref{cond:1} can be relaxed to condition~\ref{cond:1'} for $H>1$.
% Custom-labeled condition (2′), without increasing the counter
{
\renewcommand\thecondition{2$'$}
\begin{condition}\label{cond:1'}
The $\varepsilon$ is small enough that ${HL_\pi \varepsilon}/{ \min_h\min_{a_h,\tau_h^+}\pi_b(a_h|\tau_h^+)}\leq 1$.
\end{condition}
}
\setcounter{condition}{3}

\begin{theorem}[Tighter Theoretical Guarantee of FDVF]\label{Thm:tighter_fdvf_guarantee}
Suppose the abstract realizability condition $V_\CAL{F}^\phi \in \CAL{V}$ and the Bellman completeness condition $\forall V \in \CAL{V}, \CAL{B^H}V \in \Theta$ ($\CAL{B^H}$ here refers to the operator on the abstract system) hold, and Assumptions~\ref{asmp:fast_forgetting_policy} and~\ref{asmp:fast_forgetting_function_class} are satisfied. For any $\varepsilon > 0$ satisfying condition~\ref{cond:1'},~\ref{cond:2}, define $T=\max\{T_1(\varepsilon),T_2(\varepsilon)\}$.
Then, for some uniform constant $c,c_1,c_2$, with probability at least $1 - \delta$, we have: 
\begin{align*}
&|J(\pi_e)-\mathbb{E}_{\pi_b}[\hat{V}(f_1)]|\leq L_\phi\varepsilon+\sqrt{H} \cdot \\&\textstyle\qquad\max_{h\in[H]}\sup_{V\in\CAL{V}}\sqrt{\frac{\mathbb{E}_{\pi_e^\phi}[(\CAL{B}^{(\CAL{S},\CAL{H}_T)}V)(s_h,\tau_{[h-T+1:h]})^2]}{\mathbb{E}_{\pi_b^\phi}[(\CAL{B^H}V)(\tau_h)^2]}}\nonumber\\&\textstyle\qquad\qquad\quad\cdot\sqrt{\frac{cHC_{\CAL{V}}^2C_\mu}{n}\log\frac{|\CAL{V}||{\Theta}|}{\delta}+L_{\CAL{E}}\varepsilon} 
\end{align*}
where $C_{\CAL{V}}:=\max\{\|\CAL{V}\|_\infty+1,\|\Theta\|_\infty\}$, $L_\phi={R_{\rm max}HL_\pi}+{ R_{\rm max}}H^2L_\pi+\|\CAL{V}\|_\infty$ and $L_\varepsilon=3\cdot\allowbreak{}\big(\frac{HL_\pi({c_1(C_\mu+1)\|\CAL{V}\|_\infty\|\Theta\|_\infty}+{c_2H\max\{C_\mu\|\CAL{V}\|_\infty\|\Theta\|_\infty, \frac{1}{2}\|\Theta\|_\infty^2\}})}{\min_h\min_{a_h,\tau_h^+}\pi_b(a_h|\tau_h^+)}\allowbreak{}+HC_\mu\|\CAL{V}\|_\infty\|\Theta\|_\infty\big)$ 
\end{theorem}
%\yplu{have a remark here briefly discuss why your bound is better/break the cures of horizion}
\begin{remark}
    The coverage in our result only takes in a history of window $T$ instead of the entire horizon $H$, and Theorem~\ref{Thm:coverage_compare_2},~\ref{Thm:coverage_compare_3} proves that the $L_2$ and $L_\infty$ belief coverage in the belief one-hot scenario are no worse than the original. Example~\ref{example:Fast_forgetting_H_infty} also shows a polynomial finite sample guarantee while the original bound does not exist, effectively mitigating the curse of memory. Despite that structural assumption on POMDP model is adopted for Theorem~\ref{Thm:finite_sample_FDVF}, this can be avoided by taking a simpler pipeline (i.e. Theorem~\ref{Thm:tighter_fdvf_guarantee}) which gives us a even better result, indicating the advantage in tractability of memory to horizon.
\end{remark}

%\subsubsection*{Acknowledgements}
%YZ thanks Nan Jiang for useful discussions.

%\nocite{pmlr-v80-asadi18a, doi:10.1137/10080484X, shen2020deepreinforcementlearningrobust, pmlr-v206-gottesman23a,zhang2024cursesfuturehistoryfuturedependent,NEURIPS2023_3380e811}
\bibliography{bibtex}

\clearpage

%%%%%%%%%%%%%%%%%%%%%%%%%%%%%%%%%%%%%%%%%%%%%%%%%%%%%%%%%%%%
\section*{Checklist}

\begin{enumerate}

  \item For all models and algorithms presented, check if you include:
  \begin{enumerate}
    \item A clear description of the mathematical setting, assumptions, algorithm, and/or model. [Yes]
    \item An analysis of the properties and complexity (time, space, sample size) of any algorithm. [Yes]
    \item (Optional) Anonymized source code, with specification of all dependencies, including external libraries. [Not Applicable]
  \end{enumerate}

  \item For any theoretical claim, check if you include:
  \begin{enumerate}
    \item Statements of the full set of assumptions of all theoretical results. [Yes]
    \item Complete proofs of all theoretical results. [Yes]
    \item Clear explanations of any assumptions. [Yes]     
  \end{enumerate}

  \item For all figures and tables that present empirical results, check if you include:
  \begin{enumerate}
    \item The code, data, and instructions needed to reproduce the main experimental results (either in the supplemental material or as a URL). [Not Applicable]
    \item All the training details (e.g., data splits, hyperparameters, how they were chosen). [Not Applicable]
    \item A clear definition of the specific measure or statistics and error bars (e.g., with respect to the random seed after running experiments multiple times). [Not Applicable]
    \item A description of the computing infrastructure used. (e.g., type of GPUs, internal cluster, or cloud provider). [Not Applicable]
  \end{enumerate}

  \item If you are using existing assets (e.g., code, data, models) or curating/releasing new assets, check if you include:
  \begin{enumerate}
    \item Citations of the creator If your work uses existing assets. [Not Applicable]
    \item The license information of the assets, if applicable. [Not Applicable]
    \item New assets either in the supplemental material or as a URL, if applicable. [Not Applicable]
    \item Information about consent from data providers/curators. [Not Applicable]
    \item Discussion of sensible content if applicable, e.g., personally identifiable information or offensive content. [Not Applicable]
  \end{enumerate}

  \item If you used crowdsourcing or conducted research with human subjects, check if you include:
  \begin{enumerate}
    \item The full text of instructions given to participants and screenshots. [Not Applicable]
    \item Descriptions of potential participant risks, with links to Institutional Review Board (IRB) approvals if applicable. [Not Applicable]
    \item The estimated hourly wage paid to participants and the total amount spent on participant compensation. [Not Applicable]
  \end{enumerate}

\end{enumerate}

\clearpage

\appendix
\onecolumn
\section{The Unified Analysis}\label{appendix:unified}
%\paragraph{Proof of Corollary~\ref{coro:meta}}
\paragraph{Proof of Meta-theorem~\ref{Thm:meta}}
\begin{proof}
First, we have from Theorem~\ref{Thm:L_phi_1},
\begin{align*}
    |\mathfrak{est}(Q^\pi)-\mathfrak{est}^\phi(Q_\phi^\pi)|=&\;|\mathbb{E}_{b\sim d_0}[V^{\pi_\phi}_{\rm bin}(\phi(b))-V^\pi_{\rm true}(b)]|\nonumber\\
        \leq&\;\|[V^{\pi_\phi}_{\rm bin}]_{\rm true}-V^\pi_{\rm true}\|_\infty\\
        \leq&\;L_\phi^{[1]}\varepsilon.
\end{align*}
Then using triangle's inequality, we get
\begin{align*}
    |\mathfrak{est}(\hat{Q}^\pi)-\mathfrak{est}^\phi(Q^\pi)|\leq&\; |\mathfrak{est}(\hat{Q}^\pi)-\mathfrak{est}^\phi(\hat{Q}^\pi)|+|\mathfrak{est}^\phi(\hat{Q}^\pi)-\mathfrak{est}^\phi(Q_\phi^\pi)|+|\mathfrak{est}(Q^\pi)-\mathfrak{est}^\phi(Q_\phi^\pi)|\\
    \leq&\; L_\phi^{[1]}+L_\phi^{[2]}+C_\pi^\phi\cdot\sqrt{\|\CAL{V}\|_\infty\cdot\bigg(\frac{1}{n}\log\frac{|\CAL{V}|}{\delta}\bigg)^\alpha+L_{\CAL{E}}\varepsilon},\;\; w.p.>1-\delta
\end{align*}
And that completes the proof.
\end{proof}
\begin{lemma}\label{lemma:meta}
    If for any $\varepsilon$ that satisfies $D(\varepsilon)$, the following holds
    \begin{align}
        |J(\pi)-\hat{J}(\pi)|\leq L_\phi\varepsilon+ C_\pi^\phi\cdot \sqrt{\|\CAL{V}\|_\infty\cdot\bigg(\frac{1}{n}\log\frac{|\CAL{V}|}{\delta}\bigg)^\alpha+L_\CAL{E}\varepsilon}\quad w.p. \geq 1-\delta.
    \end{align}
    Then
    \begin{align}
        |J(\pi)-\hat{J}(\pi)|\leq \inf_{\substack{\varepsilon\geq 0\\D(\varepsilon)}}\Bigg(L_\phi\varepsilon+ C_\pi^\phi\cdot \sqrt{\|\CAL{V}\|_\infty\cdot\bigg(\frac{1}{n}\log\frac{|\CAL{V}|}{\delta}\bigg)^\alpha+L_\CAL{E}\varepsilon}\Bigg)\quad w.p. \geq 1-\delta.
    \end{align}
\end{lemma}
\begin{proof}
    Let $\beta(\varepsilon):=L_\phi\varepsilon+ C_\pi^\phi\cdot \sqrt{\|\CAL{V}\|_\infty\cdot\bigg(\frac{1}{n}\log\frac{|\CAL{V}|}{\delta}\bigg)^\alpha+L_\CAL{E}\varepsilon}$, and $\beta^*:=\inf_{\substack{\varepsilon\geq 0\\D(\varepsilon)}}\beta(\varepsilon)$. Then there exists a sequence of $\{\varepsilon_i\}_{i=1}^\infty$ satisfying $\varepsilon_i\geq 0$ and $D(\varepsilon_i)$, such that $\beta(\varepsilon_i)\downarrow\beta^*$. Then the family of events $\{E_i:=\{\omega\in\Omega:|J(\pi)-\hat{J}(\pi)|(\omega)\leq \beta(\varepsilon_i)\}\}_{i=1}^\infty$ is decreasing, with the limit being $E_\infty:=\{\omega\in\Omega:|J(\pi)-\hat{J}(\pi)|(\omega)\leq \beta^*\}$. It then suffice to prove the result by applying the monotone convergence theorem of measure, which shows that $\Pr(E_\infty)=\lim_{i\to\infty}\Pr(E_i)\geq 1-\delta$.
\end{proof}
\paragraph{Proof of Proposition~\ref{prop:one_hot_L_b}}
\begin{proof}
    This is because for any $b_1,b_2\in\CAL{B}, a\in\CAL{A}, o\in\CAL{O}$, if $b_1\neq b_2$, $b_1^{o,a}$ and $b_2^{o,a}$ would either be identical, thus $\frac{\|b_1^{o,a}-b_2^{o,a}\|_1}{\|b_1-b_2\|_1}=0$, or be different, thus $\frac{\|b_1^{o,a}-b_2^{o,a}\|_1}{\|b_1-b_2\|_1}=1$.
\end{proof}

\paragraph{Explanation of Example}
\begin{example}\label{example:counter-exmp}
    Consider a latent MDP with one action $a$, two states $s_1,s_2$ and three observations $o_1,o_2,o_3,o_4$. The initial state is evenly distributed over $s_1$ and $s_2$, the emission probability of $s_1$ is $(0.5,0,0.5-\xi,\xi)$ and of $s_2$ is $(0,0.5,\xi,0.5-\xi)$. Then for two belief states $b_1:=\mathbf{b}(o_2)=(0,1)$ and $b_2:=\mathbf{b}(o_4)=(2\xi,1-2\xi)$, simultaneously taking action $a$ and observing $o_3$ makes the successive belief becomes $b_1^{o_3,a}=(0,1)$ and $b_2^{o_3,a}=(0.5,0.5)$. This violates the contraction property as $\|b_1^{o_3,a}-b_2^{o_3,a}\|_1\geq\frac{1}{4\xi}\|b_1-b_2\|_1$ for any $\xi\leq\frac{1}{4}$, also showing that the Lipchitz parameter can be arbitrarily large as $\xi\to 0$. 
\end{example} %\yplu{can you change the example to parametric, to show it's possible to have $\|b_1^{a,o_3}-b_2^{a,o_3}\|_1\geq \gamma \|b_1-b_2\|_1$ for all $\gamma$ and explain this kind of MDP is hard or not}
\paragraph{Proof of Lemma~\ref{lemma:update_belief_lipchitz}.}
\begin{proof}
Fix an action $a\in\mathcal A$. For a belief $b\in\mathcal B$, define the 
joint distribution over states and observations:
\begin{align*}
P_b(s,o| a) := b(s)\,P(o| s,a), \qquad 
P_b(o| a) = \sum_s P_b(s,o| a).
\end{align*}
The posterior distribution over the current state is
\begin{align*}
\tilde b^{o,a}(s) := P(s| b,o,a) 
= \frac{P_b(s,o| a)}{P_b(o| a)}=P_b(s| o,a).
\end{align*}
The next belief after observing $o$ is
\begin{align*}
b^{o,a}(s') = \sum_{s} P(s'| s,a)\,\tilde b^{o,a}(s).
\end{align*}

For each $o\in\mathcal O$, the total variation distance contracts under the 
state transition kernel:
\begin{align*}
\|b^{o,a}_1-b^{o,a}_2\|_1
= \|\tilde b^{o,a}_1 P(\cdot| \cdot,a) 
   - \tilde b^{o,a}_2 P(\cdot| \cdot,a)\|_1
\le \|\tilde b^{o,a}_1-\tilde b^{o,a}_2\|_1.
\end{align*}
 
Taking expectation with respect to $P(\cdot| b_1,a)$ gives
\begin{align*}
\mathbb E_{o\sim P(\cdot| b_1,a)}
  [\|b^{o,a}_1-b^{o,a}_2\|_1] 
&\le 
\mathbb E_{o\sim P(\cdot| b_1,a)}
  [\|\tilde b^{o,a}_1-\tilde b^{o,a}_2\|_1] \\
&=
\sum_o P_{b_1}(o| a)\sum_s
   \big|P(s| b_1,o,a)-P(s| b_2,o,a)\big| \\
&= \sum_{s,o}\big|
   P_{b_1}(s,o| a) - P_{b_1}(o| a) P_{b_2}(s| o,a)\big|.
\end{align*}
Insert and subtract $P_{b_2}(s,o| a)=P_{b_2}(o| a) P_{b_2}(s| o,a)$, then apply the triangle inequality:
\begin{align*}
\le \sum_{s,o}\big|P_{b_1}(s,o| a)-P_{b_2}(s,o| a)\big|
   + \sum_{s,o} \big|P_{b_1}(o| a)P_{b_2}(s| o,a)-P_{b_2}(s,o| a)\big|.
\end{align*}

All three mappings
\begin{align*}
k_1^a((o',s')| s)=\mathbb{I}(s=s')P(o'| s,a), 
\quad 
k_2^a((o',s')| o)=\mathbb{I}(o=o')P_{b_2}(s'| o,a),
\quad
k_3^a(o| s)= P(o| s,a)
\end{align*}
where $\mathbb{I}(x=x')$ is the indicator function on $\CAL{X}$, are Markov kernels. By the data processing inequality for total variation,
\begin{align*}
\sum_{s,o}\big|P_{b_1}(s,o| a)-P_{b_2}(s,o| a)\big| = \bigg\|\sum_{s} k_1^a((\cdot,\cdot)| s)b_1(s)-\sum_{s} k_1^a((\cdot,\cdot)| s)b_2(s)\bigg\|_1 \le \|b_1-b_2\|_1,
\end{align*}
\begin{align*}
\sum_{s,o}\big|P_{b_1}(o| a)P_{b_2}(s| o,a)-P_{b_2}(s,o| a)\big|=&\;\bigg\|\sum_o k_2^a((\cdot,\cdot)| o)P_{b_1}(o| a)-\sum_ok_2^a((\cdot,\cdot)| o)P_{b_2}(o| a)\bigg\|_1 \\\le&\; \|P_{b_1}(\cdot| a)-P_{b_2}(\cdot| a)\|_1\\
=&\;\bigg\|\sum_s b_1(s)P(\cdot| s,a)-\sum_s b_2(s)P(\cdot| s,a)\bigg\|_1\\
=&\;\bigg\|\sum_s k_3^a(\cdot| s)b_1(s)-\sum_s k_3^a(\cdot| s)b_2(s)\bigg\|_1\\
\le&\; \|b_1-b_2\|_1
\end{align*}

Combining the above inequalities yields
\begin{align*}
\mathbb E_{o\sim P(\cdot| b_1,a)}
   [\|b^{o,a}_1-b^{o,a}_2\|_1]
   \le 2\,\|b_1-b_2\|_1.
\end{align*}

By symmetry, the same bound holds when the expectation is taken with respect 
to $P(\cdot| b_2,a)$ instead of $P(\cdot| b_1,a)$. This completes the proof.
\end{proof}

\section{Abstraction under Covering}\label{appendix:abstract}
The proof for bounding the belief abstraction error, i.e. Theorem~\ref{Thm:4-17} follows a similar idea from Theorem 9 and Proposition 48 in~\cite{subramanian2022approximate}.
\begin{definition}\label{def:epsilon-cover}
    A $\varepsilon$-cover $\mathcal{C}_{\varepsilon}$ is a subspace of the belief state space which satisfies:
    \begin{align}
        \mathcal{B}\subset\bigcup_{c\in\mathcal{C}_\varepsilon}\mathbf{B}(c,\varepsilon)
    \end{align}
    where $\mathbf{B}(c,\varepsilon)$ stands for an open ball centered at $c$ with radius $\varepsilon$. The cardinality of $\mathcal{C}_\varepsilon$ is called $\varepsilon$-covering number. For every $\varepsilon$-cover $\mathcal{C}_{\varepsilon}$, there exist a partition of the belief state space, where each $c\in\mathcal{C}_{\varepsilon}$ acts as the representation element of the bin.
\end{definition}

Building on this, we can attempt to characterize how certain important quantities behave when two belief states are sufficiently close. First, the following lemma provides a bound on the difference in expected rewards when the belief states are close.
\begin{lemma}\label{lemma:1}
    For two belief states $b_1$ and $b_2$, $\forall a\in\mathcal{A}$, we have:
    \begin{align}
        |r(b_1,a)-r(b_2,a)|\leq R_{\rm max}\|b_1-b_2\|_1.
    \end{align}
\end{lemma}
\begin{proof}
    This is easily obtained from:
    \begin{align*}
        |r(b_1,a)-r(b_2,a)|=&\ |\mathbb{E}_{s\sim b_1}[r(s,a)]-\mathbb{E}_{s\sim b_2}[r(s,a)]|\\
        =& \ |\langle r(\cdot,a),b_1-b_2\rangle|\\
        \leq&\ R_{\rm max}\|b_1-b_2\|_1.
    \end{align*}
    And it shows that when treating POMDPs as belief space MDPs, there's intrinsic smoothness within the dynamic.
\end{proof}

\begin{lemma}\label{lemma:3}
    (Lemma 2 in \cite{zhang2012covering}) For any two belief points $b_1$, $b_2$ satisfying $\|b_1-b_2\|_1\leq \varepsilon$, $\sum_o|P(o|b_1,a)-P(o|b_2,a)|\leq \|b_1-b_2\|_1 \leq \varepsilon$.
\end{lemma}
Consequently, we put forward the following proposition.
\begin{comment}
\begin{proposition}\label{prop:1}
    For policy $\pi$ satisfying Assumption \ref{asmp:lip_policy}, we have for $\forall o,a$
    \begin{align}
        |P(b_1^{o,a}|b_1)-P(b_2^{o,a}|b_2)|\leq (1+L_\pi)\|b_1-b_2\|_1.
    \end{align}
\end{proposition}
\end{comment}

\begin{proposition}\label{prop:1}
    For policy $\pi$ satisfying Assumption \ref{asmp:lip_policy}, we have for $\forall o,a$
    \begin{align}
        \sum_{o,a}|P(o|b_1,a)\pi(a|b_1)-P(o|b_1,a)\pi(a|b_2)|\leq L_\pi\|b_1-b_2\|_1.
    \end{align}
\end{proposition}
\begin{proof}
    This is a direct application of the data processing inequality. Notice that $k^{b_1}((o',a')|a)=P(o'|b_1,a)\mathbb{I}(a=a')$ is a Markov kernel, then
    \begin{align*}
        \sum_{o,a}|P(o|b_1,a)\pi(a|b_1)-P(o|b_1,a)\pi(a|b_2)|=&\;\sum_{o',a'}\bigg|\sum_{a}k^{b_1}((o',a')|a)\pi(a|b_1)-\sum_{a}k^{b_1}((o',a')|a)\pi(a|b_2)\bigg|\\
        \leq&\;\|\pi(\cdot|b_1)-\pi(\cdot|b_2)\|_1\\
        \leq&\;L_\pi \|b_1-b_2\|_1
    \end{align*}
    and the proof is done.
\end{proof}

The one-step error is easy to control, however, without model irrelevant state bastraction assumptions such as bisimulation, it is extremely difficult to control the accumulative error induced by infinite amount of steps. Fortunately, stability property of value function provides us with an alternative approach.

In the abstract MDP, the tuple $(\CAL{S}, \CAL{A}, P, R, \gamma)$ is mapped by the abstraction $\phi$ to $(\CAL{S}_\phi, \CAL{A}, P_\phi, R_\phi, \gamma)$, which means that the transition dynamics $P_\phi$ in the abstract MDP are induced by the original MDP.

Specifically, the induced $P_\phi$ satisfies that there exists a family of probability measures $\{p_x\}_{x\in\CAL{S}_\phi}$, where each $p_x$ is defined on $\phi^{-1}(x)$, such that the transition probability from $\phi(s)$ to $\phi(s')$ under action $a$, namely, $P_\phi(\phi(s')|\phi(s),a)$ in the abstract MDP can be written as: \begin{align}\label{equ:3-26}P_\phi(\phi(s')|\phi(s),a)=\mathbb{E}_{s\sim p_{\phi(s)}}[P_\phi(\phi(s')|s,a)] \end{align}

Because this characterization of $P_\phi$ relies on the existence of such a family of probability measures without specifying their exact properties, any proof involving the value function ${V}_{{\rm bin}}^{\pi_\phi}$ must treat the $\{p_x\}_{x\in\CAL{S}_\phi}$ as arbitrary.

With this understanding, we now present the following theorem, which provides an upper bound on the error between ${V}^{[\pi_\phi]_{\rm true}}_{\rm true}$ and the lifted value function $[{V}_{{\rm bin}}^{\pi_\phi}]_{\rm true}$ from the abstract MDP. Importantly, the proof of this theorem does not rely on the specific form of the measures $\{p_x\}_{x\in\CAL{S}_\phi}$.
    \begin{theorem}\label{Thm:4-17}
    If Assumption~\ref{asmp:Lipchitz_value} holds, then the error between ${V}^{[\pi_\phi]_{\rm true}}_{\rm true}$ and the lifted abstract MDP’s true value function $[{V}_{{\rm bin}}^{\pi_\phi}]_{\rm true}$ can be bounded as follows: 
        \begin{align}
            \|{V}^{[\pi_\phi]_{\rm true}}_{\rm true}-[{V}^{\pi_\phi}_{\rm bin}]_{\rm true}\|_\infty\leq \frac{(R_{\rm max}+2L_V)\varepsilon}{1-\gamma}+\frac{R_{\rm max}}{(1-\gamma)^2}\varepsilon
        \end{align}
    \end{theorem}
    \begin{proof}
    We begin by clarifying and establishing the notation used in the proof. Fix an arbitrary family $\{p_x\}_{x\in\CAL{S}_\phi}$, and let $b' \sim {\rm bin}(\phi(b))$ denote the expectation taken over the following sampling process:

1. Since $\phi(b) \in \CAL{B}_\phi$ is the representative element of some partition of the belief space after binning, the set $\phi^{-1}(\phi(b)) \subset \CAL{B}$ is the corresponding element in the original belief space—i.e., the subset consisting of all belief states that are grouped into the same bin as $b$.

2. Sample a temporary belief state $b_{\rm temp}$ from $\phi^{-1}(\phi(b))$ according to the fixed distribution $p_{\phi(b)}$.

3. Starting from $b_{\rm temp}$, perform the belief update procedure, where the action $a$ is determined by the policy $\pi$. Once the update is complete, the resulting belief state is the sampled $b'$.

With this notation established, we can proceed with the proof of the theorem. The main idea of the proof is to construct a chain rule argument. First, notice that
        \begin{align}
            [{V}^{\pi_\phi}_{\rm bin}]_{\rm true}(b)=\mathbb{E}_{\substack{
              b_1\sim {\rm bin}(\phi(b)) \\
              b_2\sim {\rm bin}(\phi(b_1)) \\
              \cdots
            }}
            [r_\phi(\phi(b_1),a_1)+\gamma r_\phi(\phi(b_2),a_2)+\gamma^2r_\phi(\phi(b_3),a_3)+\cdots]
        \end{align}
        Consider $V^{[k]}$ as
        \begin{align}
            V^{[k]}=\mathbb{E}_{\substack{
              b_1\sim {\rm bin}(\phi(b)) \\
              b_2\sim {\rm bin}(\phi(b_1)) \\
              \cdots \\
              b_k\sim {\rm bin}(\phi(b_{k-1})) \\
              b_{k+1} \sim b_k \\
              \cdots
            }}
            [r_\phi(\phi(b_1),a_1)+\gamma r_\phi(\phi(b_2),a_2)+\gamma^2r_\phi(\phi(b_3),a_3)+\cdots]
        \end{align}
        Then $V^{[0]}(b)={V}^{[\pi_\phi]_{\rm true}}_{\rm true}(b)$.
        Next, for $\forall b$,
        \begingroup
        \allowdisplaybreaks
        \begin{align}
            &\ |V^{[k+1]}(b)-V^{[k]}(b)|\\
            =&\ \Bigg|\mathbb{E}_{\substack{
              b_1\sim {\rm bin}(\phi(b)) \\
              b_2\sim {\rm bin}(\phi(b_1)) \\
              \cdots \\
              b_k\sim {\rm bin}(\phi(b_{k-1})) \\
              b_{k+1} \sim b_k \\
              b_{k+2} \sim b_{k+1} \\
              \cdots
            }}
            [\gamma^{k}V_{\rm true}^{[\pi_\phi]_{\rm true}}(b_{k+1})]-
            \mathbb{E}_{\substack{
              b_1\sim {\rm bin}(\phi(b)) \\
              b_2\sim {\rm bin}(\phi(b_1)) \\
              \cdots \\
              b_k\sim {\rm bin}(\phi(b_{k-1})) \\
              b_{k+1} \sim {\rm bin}(\phi(b_k)) \\
              b_{k+2} \sim b_{k+1} \\
              \cdots
            }}
            [\gamma^kr_\phi(\phi(b_{k+1}),a)+\gamma^{k+1}V_{\rm true}^{[\pi_\phi]_{\rm true}}(b_{k+2})]\Bigg|\\
            =&\ \Bigg|\mathbb{E}_{\substack{
              b_1 \sim {\rm bin}(\phi(b)) \\
              b_2 \sim {\rm bin}(\phi(b_1)) \\
              \cdots \\
              b_k \sim {\rm bin}(\phi(b_{k-1})) \\
              b_{k+1} \sim b_k \\
              b_{k+2} \sim b_{k+1} \\
              \cdots
            }}
            [\gamma^{k}V_{\rm true}^{[\pi_\phi]_{\rm true}}(b_{k+1})]-
            \mathbb{E}_{\substack{
              b_1 \sim {\rm bin}(\phi(b)) \\
              b_2 \sim {\rm bin}(\phi(b_1)) \\
              \cdots \\
              b_k \sim {\rm bin}(\phi(b_{k-1})) \\
              b_{k+1} \sim {\rm bin}(\phi(b_k)) \\
              b_{k+2} \sim b_{k+1} \\
              \cdots
            }}
            [\gamma^kr_\phi(\phi(b_{k+1}),a)-\gamma^kr(b_{k+1},a)+\gamma^{k}V_{\rm true}^{[\pi_\phi]_{\rm true}}(b_{k+1})]\Bigg|\\
            \leq&\ \gamma^kR_{\rm max}\varepsilon+{\gamma^k}2L_V\varepsilon+\gamma^k\frac{R_{\rm max}}{1-\gamma}\varepsilon
        \end{align}
        \endgroup
        where the last inequality used the stability of value function (Assumption~\ref{asmp:Lipchitz_value}), Lemma~\ref{lemma:update_belief_lipchitz} and Lemma~\ref{lemma:1}, \ref{lemma:3} since the next belief is sampled from the start of same bin and thus close enough. Specifically, it uses the fact that
        \begingroup
        \allowdisplaybreaks
        \begin{align*}
            &\ \big|\mathbb{E}_{\substack{a\sim \pi(\phi(b_k))\\o\sim P(\cdot|b_k,a)}}[V_{\rm true}^{[\pi_\phi]_{\rm true}}(b_k^{o,a})]-\mathbb{E}_{\substack{a\sim \pi(\phi(b_k))\\b_t\sim p_{\phi(b_k)}\\o\sim P(\cdot|b_t,a)}}[V_{\rm true}^{[\pi_\phi]_{\rm true}}(b_t^{o,a})]\big|\\
            =&\ \bigg|\mathbb{E}_{b_t\sim p_{\phi(b_k)}}\bigg[\mathbb{E}_{\substack{a\sim \pi(\phi(b_k))\\o\sim P(\cdot|b_k,a)}}[V_{\rm true}^{[\pi_\phi]_{\rm true}}(b_k^{o,a})]-\mathbb{E}_{\substack{a\sim \pi(\phi(b_k))\\o\sim P(\cdot|b_t,a)}}[V_{\rm true}^{[\pi_\phi]_{\rm true}}(b_t^{o,a})]\bigg]\bigg|\\
            \leq&\ \bigg|\mathbb{E}_{b_t\sim p_{\phi(b_k)}}\bigg[\mathbb{E}_{\substack{a\sim \pi(\phi(b_k))\\o\sim P(\cdot|b_k,a)}}[V_{\rm true}^{[\pi_\phi]_{\rm true}}(b_k^{o,a})]-\mathbb{E}_{\substack{a\sim \pi(\phi(b_k))\\o\sim P(\cdot|b_k,a)}}[V_{\rm true}^{[\pi_\phi]_{\rm true}}(b_t^{o,a})]\bigg]\bigg|\\
            &\qquad\qquad\qquad\quad+
            \bigg|\mathbb{E}_{b_t\sim p_{\phi(b_k)}}\bigg[\mathbb{E}_{\substack{a\sim \pi(\phi(b_k))\\o\sim P(\cdot|b_k,a)}}[V_{\rm true}^{[\pi_\phi]_{\rm true}}(b_t^{o,a})]-\mathbb{E}_{\substack{a\sim \pi(\phi(b_k))\\o\sim P(\cdot|b_t,a)}}[V_{\rm true}^{[\pi_\phi]_{\rm true}}(b_t^{o,a})]\bigg]\bigg|\\
            \leq&\ \bigg|\mathbb{E}_{b_t\sim p_{\phi(b_k)}}\bigg[\mathbb{E}_{\substack{a\sim \pi(\phi(b_k))\\o\sim P(\cdot|b_k,a)}}[L_V\|b_k^{o,a}-b_t^{o,a}\|_1]\bigg]\bigg|\\&\qquad\qquad\qquad+\bigg|\mathbb{E}_{\substack{b_t\sim p_{\phi(b_k)}\\a\sim\pi(\phi(b_k))}}\bigg[
            \langle P(\cdot|b_k,a)-P(\cdot|b_t,a),V_{\rm true}^{[\pi_\phi]_{\rm true}}(b_t^{\cdot,a}) \rangle
            \bigg]\bigg|\\
            \leq&\ 2L_V\varepsilon+\frac{R_{\rm max}}{1-\gamma}\varepsilon
        \end{align*}
        \endgroup

        Finally, we do the telescoping, and sums up all the $V^{[k+1]}-V^{[k]}$ to get for $\forall b$,
        \begin{align}
            &\ |{V}^{[\pi_\phi]_{\rm true}}_{\rm true}(b)-[{V}^{\pi_\phi}_{\rm bin}]_{\rm true}(b)|\\
            =&\ \bigg|\sum_{k=0}^\infty\big(V^{[k+1]}(b)-V^{[k]}(b)\big)\bigg|\\
            \leq&\ \sum_{k=0}^\infty\bigg|\gamma^kR_{\rm max}\varepsilon+{\gamma^k}2L_V\varepsilon+\gamma^k\frac{R_{\rm max}}{1-\gamma}\varepsilon\bigg|\\
            \leq&\ \frac{(R_{\rm max}+2L_V)\varepsilon}{1-\gamma}+\frac{R_{\rm max}}{(1-\gamma)^2}\varepsilon
        \end{align}
    \end{proof}

    Before ending this part, we'll need to fill the gap between the target policy and the abstracted policy to which the target policy descended. This is handled by the following theorem, which does not rely on any assumption on the POMDP model itself.
    \begin{theorem}\label{Thm:policy_abstract_error}
If Assumption~\ref{asmp:lip_policy} holds.
    \begin{align}
        \|V_{\rm true}^{\pi}-V_{\rm true}^{[\pi_\phi]_{\rm true}}\|_\infty&\leq \frac{R_{\rm max}L_\pi\varepsilon}{1-\gamma}+\frac{\gamma R_{\rm max}}{(1-\gamma)^2}L_\pi\varepsilon
    \end{align}
    \end{theorem}
\begin{proof}
    Using the fact that $V_{\rm true}^{[\pi_\phi]_{\rm true}}=\mathcal{T}^{[\pi_\phi]_{\rm true}} V_{\rm true}^{[\pi_\phi]_{\rm true}}$,
    \begin{align}
        \|V_{\rm true}^{\pi}-V_{\rm true}^{[\pi_\phi]_{\rm true}}\|_\infty=&\ \|V_{\rm true}^{\pi}-\mathcal{T}^{[\pi_\phi]_{\rm true}} V_{\rm true}^{\pi}+\mathcal{T}^{[\pi_\phi]_{\rm true}} V_{\rm true}^{\pi}-\mathcal{T}^{[\pi_\phi]_{\rm true}} V_{\rm true}^{[\pi_\phi]_{\rm true}}\|_\infty    \nonumber\\
        \leq&\ \|V_{\rm true}^{\pi}-\mathcal{T}^{[\pi_\phi]_{\rm true}} V_{\rm true}^{\pi}\|_\infty+\gamma\|V_{\rm true}^{\pi}-V_{\rm true}^{[\pi_\phi]_{\rm true}}\|_\infty.
    \end{align} 
    Here for an MDP $(\CAL{S},\CAL{A},r,\gamma,P)$, $(\CAL{T}^\pi V)(s):=\mathbb{E}_{\substack{a'\sim\pi(\cdot|s)\\s'\sim P(\cdot|s,a')}}[r(s,a')+\gamma V(s')]$ is the Bellman operator, which is a $\gamma$-Lipchitz compression operator w.r.t. the infinity norm. Consequently,
    \begin{align}
        \|V_{\rm true}^{\pi}-V_{\rm true}^{[\pi_\phi]_{\rm true}}\|_\infty\leq \frac{1}{1-\gamma}\|V_{\rm true}^{\pi}-\mathcal{T}^{[\pi_\phi]_{\rm true}} V_{\rm true}^{\pi}\|_\infty.
    \end{align}
    For any $b$, we have
    \begin{align}
        &\ |(V_{\rm true}^{\pi}-\mathcal{T}^{[\pi_\phi]_{\rm true}} V_{\rm true}^{\pi})(b)|\nonumber\\
        =&\ |(\mathcal{T}^\pi V_{\rm true}^{\pi}-\mathcal{T}^{[\pi_\phi]_{\rm true}} V_{\rm true}^{\pi})(b)|\nonumber\\
        =&\ \bigg|\mathbb{E}_{\substack{a\sim\pi(b)\\ b^{+1}\sim P(\cdot|b)}}\bigg[r+\gamma V_{\rm true}^\pi(b^{+1})\bigg]-\mathbb{E}_{\substack{a\sim\pi_\phi(\phi(b))\\ b^{+1}\sim P(\cdot|b)}}\bigg[r+\gamma V_{\rm true}^\pi(b^{+1})\bigg]\bigg|
    \end{align}
    We first look at $r$,
    \begin{align}
        |\mathbb{E}_{a\sim\pi(b)}[r]-\mathbb{E}_{a\sim\pi_\phi(\phi(b))}[r]|
        =&\ |\mathbb{E}_{a\sim\pi(b)}[r]-\mathbb{E}_{a\sim \pi(\phi(b))}[r]|\nonumber\\
        \leq&\ R_{\rm max}L_\pi\varepsilon
    \end{align}
    Then we look at $V^\pi_{\rm true}$,
    \begin{align}
        &\ \bigg|\mathbb{E}_{\substack{a\sim\pi(b)\\ b^{+1}\sim P(\cdot|b)}}\bigg[\gamma V_{\rm true}^\pi(b^{+1})\bigg]-\mathbb{E}_{\substack{a\sim\pi_\phi(\phi(b))\\ b^{+1}\sim P(\cdot|b)}}\bigg[\gamma V_{\rm true}^\pi(b^{+1})\bigg]\bigg|\nonumber\\
        =&\ \gamma\bigg|\sum_{o\in\mathcal{O}}\sum_{a\in\mathcal{A}}
        \bigg[\big(P(o|b,a)\pi(a|b)-P(o|b,a)\pi(a|\phi(b))\big)\cdot V^\pi(b^{o,a})\bigg]\bigg|\nonumber\\
        \leq&\ \frac{\gamma R_{\rm max}}{1-\gamma}L_\pi\varepsilon
    \end{align}
    where we used Proposition \ref{prop:1} for the final inequality.
    \end{proof}

\paragraph{Proof of Theorem~\ref{Thm:L_phi_1}.}
\begin{proof} Combining Theorem~\ref{Thm:4-17},~\ref{Thm:policy_abstract_error}, we get
\begin{align}
     \|[V^{\pi_\phi}_{\rm bin}]_{\rm true}-V^\pi_{\rm true}\|_\infty
        \leq&\ \|[V^{\pi_\phi}_{\rm bin}]_{\rm true}-V_{{\rm true}}^{[\pi_\phi]_{{\rm true}}}\|_\infty+\|V_{{\rm true}}^{[\pi_\phi]_{{\rm true}}}-
        V^\pi_{\rm true}\|_\infty\nonumber\\
        \leq&\ \frac{(L_\pi+1)R_{\rm max}+2L_V}{1-\gamma}\varepsilon
        +\frac{\gamma R_{\rm max}L_\pi+R_{\rm max}}{(1-\gamma)^2}\varepsilon,
\end{align}
thereby completing the proof.
\end{proof}

\section{Double Sampling Analysis in Chapter~\ref{Chapter:double_sampling}}\label{appendix:dbsp}
\begin{definition}[abstract algorithm]
Consider the Bellman error minimization algorithm using double sampling, not only is it executed in the real world (simulator), but also virtually in the abstract system, using the same piece of offline data. The optimization target for the abstract algorithm can be written as
    \begin{align}
        \hat{Q}_{\phi}^{\pi_\phi}=\argmin_{f\in\mathcal{F}}\mathcal{E}_\phi(f,\pi)
    \end{align}
    where%\zyh{move to appendix, change to define covering}
    \begin{align}
        \mathcal{E}_\phi(f,\pi)=&\ \mathbb{E}_{\mathcal{D}}[(f(\phi(b),a)-(r_\phi+\gamma f(\phi(b'_A),\pi_\phi)))(f(\phi(b),a)-(r_\phi+\gamma f(\phi(b'_B),\pi_\phi)))].
    \end{align}
\end{definition}
\begin{lemma}[MDP telescoping~\cite{jiang2024offline_survey}]\label{lem:telescoping}
    For an MDP $(\CAL{S},\CAL{A},r,\gamma,P)$ and any function $Q:\CAL{S}\times\CAL{A}\to\R^+$, we have
    \begin{align}
        J_{Q}(\pi)-J(\pi)=\frac{1}{1-\gamma}\mathbb{E}_{d^\pi}[Q-\CAL{T}^\pi Q]
    \end{align}
    where $(\CAL{T}^\pi Q)(s,a):=r(s,a)+\gamma\mathbb{E}_{\substack{s'\sim P(\cdot|s,a)\\a'\sim\pi(\cdot|s')}}[Q(s',a')]$ is the Bellman operator.
\end{lemma}

%\subsection{Lemma~\ref{lem:5-3} and the proof}

\begin{lemma}\label{lem:5-3}
In the binned system, we have the following telescoping error
    \begin{align}
        |J_{\hat{Q}}(\pi_\phi)-J(\pi_\phi)|\leq \frac{\sqrt{C_\pi(\phi)}}{1-\gamma}\cdot\sqrt{\mathbb{E}_{d^{D}}[(\hat{Q}-\mathcal{T}^{\pi_\phi}\hat{Q})^2]}
    \end{align}
\end{lemma}
\begin{proof}
Recall the previously mentioned Lemma~\ref{lem:telescoping}. Substituting it into the case of the abstract belief MDP gives:
\begin{align}
J_{\hat{Q}}(\pi_\phi)-J(\pi_\phi)=\frac{1}{1-\gamma}\mathbb{E}_{d^{\pi_\phi}}[\hat{Q}-\CAL{T}^{\pi_\phi}\hat{Q}]
\end{align}
Therefore, we have:
\begin{align}
|J_{\hat{Q}}(\pi_\phi)-J(\pi_\phi)|=&\ \bigg|\frac{1}{1-\gamma}\mathbb{E}_{d^{\pi_\phi}}[\hat{Q}-\CAL{T}^{\pi_\phi}\hat{Q}]\bigg|\nonumber\\
        \leq&\ \frac{1}{1-\gamma}\mathbb{E}_{d^{\pi_\phi}}[|\hat{Q}-\CAL{T}^{\pi_\phi}\hat{Q}|]\\
        \leq&\ \frac{1}{1-\gamma}\sqrt{\mathbb{E}_{d^{\pi_\phi}}[(\hat{Q}-\CAL{T}^{\pi_\phi}\hat{Q})^2]}\\
        \leq&\ \frac{1}{1-\gamma}\sqrt{\mathbb{E}_{d^{D}}[\frac{d^\pi}{d^D}(\hat{Q}-\CAL{T}^{\pi_\phi}\hat{Q})^2]}\\
        \leq&\ \frac{\sqrt{C_\pi(\phi)}}{1-\gamma}\cdot\sqrt{\mathbb{E}_{d^{D}}[(\hat{Q}-\mathcal{T}^{\pi_\phi}\hat{Q})^2]}
\end{align}
\end{proof}
%\zyh{The proof is standard textbook so I omitted it here.}
And we obviously have
\begin{align}\label{equ:5-11}
    \mathbb{E}_{d^{D}}[\mathcal{E}_\phi(\hat{Q},\pi)]=\mathbb{E}_{d^{D}}[(\hat{Q}-\mathcal{T}^{\pi_\phi}\hat{Q})^2]
\end{align}

%\subsection{Lemma~\ref{lem:hoeffding_result} and the proof}

As the size of independent samples grows, the difference between the empirical estimate and the true expectation of the value above becomes closer, whose convergence speed can be characterized using concentration inequalities such as Hoeffding's or Bernstein's inequality. Using Hoeffding's inequality, we get the following lemma.
\begin{lemma}\label{lem:hoeffding_result}
    With probability at least $1-\delta$, for $\forall f\in\mathcal{F}$,
    \begin{align}
        |\mathcal{E}_\phi(f,\pi)-\mathbb{E}_{d^{D}}[\mathcal{E}_\phi(f,\pi)]|\leq\sqrt{\frac{8R_{\rm max}^4}{n(1-\gamma)^4}\cdot\log\frac{2|\mathcal{F}|}{\delta}}
    \end{align}
\end{lemma}
\begin{proof}
For $\CAL{E}_\phi(f,\pi)$, we first estimate an upper bound on its absolute value. Since $f \in \CAL{F}$ is used to approximate a value function, its upper bound can be assumed to be no greater than $R{\rm max}/(1 - \gamma)$, i.e., the upper bound of the value function. Therefore, we can give a rough upper bound (possibly with a constant slack, which is acceptable since it’s only a constant):
\begin{align*}
0\leq \CAL{E}_\phi(f,\pi)\leq \frac{4R_{\rm max}^2}{(1-\gamma)^2}
\end{align*}

Thus, by Hoeffding’s inequality, for any $f \in \CAL{F}$, we have:
\begin{align}
&\ \Pr(|\mathcal{E}_\phi(f,\pi)-\mathbb{E}_{d^{D}}[\mathcal{E}_\phi(f,\pi)]|\geq t)\leq 2\exp\bigg(-\frac{2t^2n(1-\gamma)^4}{16R_{\rm max}^4}\bigg)\nonumber\\
        \to&\ \Pr(|\mathcal{E}_\phi(f,\pi)-\mathbb{E}_{d^{D}}[\mathcal{E}_\phi(f,\pi)]|> t)\leq 2\exp\bigg(-\frac{2t^2n(1-\gamma)^4}{16R_{\rm max}^4}\bigg)\label{ineq:hoeffding_target}
\end{align}

However, the goal of the proof is actually:
\begin{align}\label{equ:hoeffding_target_ultimate}
\Pr(\forall f\in\CAL{F}, |\mathcal{E}_\phi(f,\pi)-\mathbb{E}_{d^{D}}[\mathcal{E}_\phi(f,\pi)]|\leq t)
\end{align}

For such problems, a common approach is to use the union bound. Let the probability space be $(\Omega, \Sigma, \Pr)$, and define the events:
\begin{align*}
A:=&\ \{\omega\in\Omega:\forall f\in\CAL{F},|\mathcal{E}_\phi(f,\pi)-\mathbb{E}_{d^{D}}[\mathcal{E}_\phi(f,\pi)]|(\omega)\leq t\}\\
        B_f:=&\ \{\omega\in\Omega: |\mathcal{E}_\phi(f,\pi)-\mathbb{E}_{d^{D}}[\mathcal{E}_\phi(f,\pi)]|(\omega)> t\}
\end{align*}

Then:
\begin{align}
&\ \Pr(\forall f\in\CAL{F}, |\mathcal{E}_\phi(f,\pi)-\mathbb{E}_{d^{D}}[\mathcal{E}_\phi(f,\pi)]|\leq t)\nonumber\\
        =&\ \Pr(A)\nonumber
        = 1-\Pr(\Omega\backslash A)
        = 1-\Pr(\bigcup_{f\in\CAL{F}}B_f)\\
        \geq&\ 1-\sum_{f\in\CAL{F}}\Pr(B_f)
        \geq 1-2|\CAL{F}|\exp \bigg(-\frac{t^2n(1-\gamma)^4}{8R_{\rm max}^4}\bigg)
\end{align}

The second-to-last step uses the subadditivity of probability (countable subadditivity), and the final step applies inequality~\eqref{ineq:hoeffding_target}.

Let $\delta = 2|\CAL{F}|\exp \big(-{t^2n(1-\gamma)^4}/{8R_{\rm max}^4}\big)$, then solving for $t$ gives $t = \sqrt{\frac{8R_{\rm max}^4}{n(1-\gamma)^4}\cdot\log\frac{2|\CAL{F}|}{\delta}}$

Substituting this into~\eqref{equ:hoeffding_target_ultimate} completes the proof.
\end{proof}

In fact, Hoeffding's inequality only leverages the boundedness of the function. However, by introducing the Bellman completeness assumption below, we can also take the variance of the function into account and apply Bernstein's inequality to achieve a tighter convergence rate.

And the standard Bellman completeness assumption is as below:
\begin{assumption}[Bellman Completeness]\label{asmp:bellman_completeness}
 $\forall f\in\mathcal{F}, \mathcal{T}^{\pi_\phi} f\in\mathcal{F}$.
\end{assumption}

\begin{remark}
For a finite function space where $|\CAL{F}| < \infty$, the Bellman completeness Assumption~\ref{asmp:bellman_completeness} implies the realizability Assumption~\ref{asmp:realizability}.
\end{remark}
\begin{proposition}\label{prop:bernstein_result}
Under the Bellman completeness Assumption~\ref{asmp:bellman_completeness}, we can obtain an upper bound with $O(1/n)$ convergence rate. Specifically, with probability at least $1 - \delta$, for all $f \in \CAL{F}$, the following holds:
\begin{align}
|\mathcal{E}_\phi(f,\pi)-\mathbb{E}_{d^{D}}[\mathcal{E}_\phi(f,\pi)]|\lesssim \frac{R_{\rm max}^2}{n(1-\gamma)^2}\cdot\log\frac{|\mathcal{F}|}{\delta}
\end{align}
\end{proposition}

Of course, for the purpose of this discussion, the Bellman completeness assumption is not necessary—only the following realizability assumption is needed to achieve the goal. However, in this case, we can only characterize the concentration rate using Lemma~\ref{lem:hoeffding_result} derived from Hoeffding’s inequality, and cannot use the tighter concentration rate provided by Proposition~\ref{prop:bernstein_result}.
\begin{lemma}\label{lem:6}
If Assumption~\ref{asm:lipchitz_function_class} holds, then
    \begin{align}
        |\mathcal{E}(f,\pi)-\mathcal{E}_\phi(f,\pi)|\leq \frac{4R_{\rm max}}{1-\gamma}\cdot\bigg((1+\gamma)L_Q+\frac{R_{\rm max}}{1-\gamma}\bigg)\varepsilon
    \end{align}
    %\zyh{move to appendix}
\end{lemma}
%\subsection{Proof of Lemma~\ref{lem:6}}
\begin{proof}
    \begin{align}
        &\ |\mathcal{E}(f,\pi)-\mathcal{E}_\phi(f,\pi)|\nonumber\\
        =&\ |\mathbb{E}_{\mathcal{D}}[(f(b,a)-(r+\gamma f(b'_A,\pi)))(f(b,a)-(r+\gamma f(b'_B,\pi)))]-\nonumber\\
        &\quad\mathbb{E}_{\mathcal{D}}[(f(\phi(b),a)-(r_\phi+\gamma f(\phi(b'_A),\pi_\phi)))(f(\phi(b),a)-(r_\phi+\gamma f(\phi(b'_B),\pi_\phi)))]|\nonumber\\
        \leq&\ |\mathbb{E}_{\mathcal{D}}[\{(f(b,a)-f(\phi(b),a))-(r(b,a)-r_\phi(\phi(b),a))-\gamma (f(b'_A,\pi)-f(\phi(b'_A),\pi_\phi)\}\nonumber\\
        &\quad\quad\quad\cdot(f(b,a)-(r+\gamma f(b'_B,\pi)))]|+\nonumber\\
        &\ |\mathbb{E}_{\mathcal{D}}[\{(f(b,a)-f(\phi(b),a))-(r(b,a)-r_\phi(\phi(b),a))-\gamma (f(b'_B,\pi)-f(\phi(b'_B),\pi_\phi)\}\nonumber\\
        &\quad\quad\quad\cdot(f(b,a)-(r+\gamma f(b'_A,\pi)))]|.
    \end{align}
    Using the fact that
    \begin{align}
        &\ |f(b,\pi)-f(\phi(b),\pi_\phi)|\nonumber\\
        =&\ |\mathbb{E}_{\pi(a|b)}[f(b,a)]-\mathbb{E}_{\pi(a|\phi(b))}[f(\phi(b),a)]|\nonumber\\
        \leq&\ |\mathbb{E}_{\pi(a|b)}[f(b,a)]-\mathbb{E}_{\pi(a|\phi(b))}[f(b,a)]|+|\mathbb{E}_{\pi(a|\phi(b))}[f(b,a)]-\mathbb{E}_{\pi(a|\phi(b))}[f(\phi(b),a)]|\nonumber\\
        \leq&\ \frac{R_{\rm max}}{1-\gamma}\varepsilon+L_Q\varepsilon
    \end{align}
    we have
    \begin{align}
        &\ |\mathcal{E}(f,\pi)-\mathcal{E}_\phi(f,\pi)|\nonumber\\
        \leq&\ 2\cdot\frac{2R_{\rm max}}{1-\gamma}\cdot\bigg((1+\gamma)L_Q+\frac{R_{\rm max}}{1-\gamma}\bigg)\varepsilon
    \end{align}
\end{proof}

\begin{proposition}\label{prop:5-16}
If Assumption~\ref{asmp:realizability} ,~\ref{asm:lipchitz_function_class} holds, then
    \begin{align}
        |\mathbb{E}_{d^{D}}[\mathcal{E}_\phi(\hat{Q}^\pi,\pi)]|\leq \sqrt{\frac{32R_{\rm max}^4}{n(1-\gamma)^4}\cdot\log\frac{2|\mathcal{F}|}{\delta}}+\frac{8R_{\rm max}}{1-\gamma}\cdot\bigg((1+\gamma)L_Q+\frac{R_{\rm max}}{1-\gamma}\bigg)\varepsilon
    \end{align}
%    \zyh{move to appendix}
\end{proposition}
\begin{proof}
    Using Lemma~\ref{lem:hoeffding_result}, we have
    \begin{align}
        |\mathcal{E}_\phi(\hat{Q},\pi)-\mathbb{E}_{d^{D}}[\mathcal{E}_\phi(\hat{Q},\pi)]|\leq\sqrt{\frac{8R_{\rm max}^4}{n(1-\gamma)^4}\cdot\log\frac{2|\mathcal{F}|}{\delta}}
    \end{align}
    and
    \begin{align}
        |\mathcal{E}_\phi({{Q}_\phi^{\pi_\phi}},\pi)-\mathbb{E}_{d^{D}}[\mathcal{E}_\phi({{Q}_\phi^{\pi_\phi}},\pi)]|\leq\sqrt{\frac{8R_{\rm max}^4}{n(1-\gamma)^4}\cdot\log\frac{2|\mathcal{F}|}{\delta}}
    \end{align}
    where 
$
\mathbb{E}_{d^{D}}[\mathcal{E}_\phi({Q}_\phi^{\pi_\phi},\pi)] = 0
$.
Then using Lemma~\ref{lem:6}, we have with probability greater than $1-\delta$
\begin{align}
    |\CAL{E}(\hat{Q},\pi)-\mathbb{E}_{d^D}[\CAL{E}_\phi(\hat{Q},\pi)]|\leq \sqrt{\frac{8R_{\rm max}^4}{n(1-\gamma)^4}\cdot\log\frac{2|\mathcal{F}|}{\delta}}+\frac{4R_{\rm max}}{1-\gamma}\cdot\bigg((1+\gamma)L_Q+\frac{R_{\rm max}}{1-\gamma}\bigg)\varepsilon
\end{align}
and
\begin{align}
    |\CAL{E}({Q}_\phi^{\pi_\phi},\pi)-\mathbb{E}_{d^D}[\CAL{E}_\phi({Q}_\phi^{\pi_\phi},\pi)]|\leq \sqrt{\frac{8R_{\rm max}^4}{n(1-\gamma)^4}\cdot\log\frac{2|\mathcal{F}|}{\delta}}+\frac{4R_{\rm max}}{1-\gamma}\cdot\bigg((1+\gamma)L_Q+\frac{R_{\rm max}}{1-\gamma}\bigg)\varepsilon
\end{align}
Using the abstract realizability Assumption~\ref{asmp:realizability}, we have $\CAL{E}(\hat{Q},\pi)=\min_{f\in\CAL{F}}\CAL{E}(f,\pi)\leq \CAL{E}({Q}_\phi^{\pi_\phi},\pi)$, and consequently
\begin{align}
    |\mathbb{E}_{d^{D}}[\mathcal{E}_\phi(\hat{Q}^\pi,\pi)]|\leq \sqrt{\frac{32R_{\rm max}^4}{n(1-\gamma)^4}\cdot\log\frac{2|\mathcal{F}|}{\delta}}+\frac{8R_{\rm max}}{1-\gamma}\cdot\bigg((1+\gamma)L_Q+\frac{R_{\rm max}}{1-\gamma}\bigg)\varepsilon.
\end{align}
\end{proof}

%With the assumption on the function class, we can therefore control the differences between $\mathcal{E}_\phi(f,\pi)$ and $\mathcal{E}(f,\pi)$ for the very same fixed $f\in\mathcal{F}$, which is stated in Lemma~\ref{lem:6} in the appendix.
%\zyh{Before that, there's some problem with common abstraction literature that I'd like to point out. There's a little difference in setting between my Theorem \ref{Thm:abstraction} and standard abstraction. I'll draw a graph to indicate the subtle relation and difference.}

%To put things together, we have

And consequently,
\begin{theorem}\label{Thm:6.}
If Assumption~\ref{asmp:realizability},~\ref{asm:lipchitz_function_class} hold, then
    \begin{align}
        |J_{\hat{Q}^\pi}(\pi_\phi)&-J(\pi_\phi)|\leq\frac{\sqrt{C_\pi(\phi)}}{1-\gamma} \cdot
        \sqrt{\sqrt{\frac{32R_{\rm max}^4}{n(1-\gamma)^4}\cdot\log\frac{2|\mathcal{F}|}{\delta}}+L_\CAL{E}\varepsilon}
    \end{align}  
    where
    \begin{align}
        L_\CAL{E}:=\frac{8R_{\rm max}}{1-\gamma}\cdot\bigg((1+\gamma)L_Q+\frac{R_{\rm max}}{1-\gamma}\bigg)
    \end{align}
\end{theorem}
\begin{proof}
The result follows directly from Proposition~\ref{prop:5-16},  Lemma~\ref{lem:5-3}, and~\eqref{equ:5-11}.
\end{proof}

%\subsection{Theorem~\ref{Thm:result_of_sect3},~\ref{Thm:8} and the proof}
The following series of theorems are all preparatory steps toward ultimately controlling the overall error.

% \begin{theorem}\label{Thm:result_of_sect3}
% If Assumptions~%\ref{asmp:injection}
% \ref{asmp:lip_policy}, and\ref{asmp:Lipchitz_value} hold, then we have:
% \begin{align}
%         |J(\pi_\phi)-J(\pi)|\leq&\ \frac{(L_\pi+1)R_{\rm max}+2|\CAL{O}|L_V}{1-\gamma}\varepsilon
%         +\frac{\gamma |\mathcal{O}||\mathcal{A}|R_{\rm max}L_\pi+|\CAL{O}|R_{\rm max}}{(1-\gamma)^2}\varepsilon.
% \end{align}
% \end{theorem}
% \begin{proof} We have
%     \begin{align}
%         &\ |J(\pi_\phi)-J(\pi)|\nonumber\\=&\ |\mathbb{E}_{b\sim d_0}[V^{\pi_\phi}_{\rm bin}(\phi(b))-V^\pi_{\rm true}(b)]|\nonumber\\
%         \leq&\ \|[V^{\pi_\phi}_{\rm bin}]_{\rm true}-V^\pi_{\rm true}\|_\infty
%     \end{align}
%     Then the result follows from Theorem~\ref{Thm:L_phi_1}.
% \end{proof}

\begin{theorem}\label{Thm:8}
If Assumption~\ref{asm:lipchitz_function_class} holds, then
    \begin{align}
        |J_{\hat{Q}^\pi}(\pi)-J_{\hat{Q}^\pi}(\pi_\phi)|\leq \frac{R_{\rm max}}{1-\gamma}\varepsilon+L_Q\varepsilon
    \end{align}
%\zyh{move to appendix}
\end{theorem}
\begin{proof}
We have
    \begin{align}
        &\ |J_{\hat{Q}^\pi}(\pi)-J_{\hat{Q}^\pi}(\pi_\phi)|\nonumber\\
        =&\ |\mathbb{E}_{b\sim d_0}[\hat{Q}^\pi(b,\pi)]-\mathbb{E}_{b\sim d_0}[\hat{Q}^\pi(\phi(b),\pi_\phi)]|\nonumber\\
        =&\ |\mathbb{E}_{b\sim d_0}[\hat{Q}^\pi(b,\pi)-\hat{Q}^\pi(\phi(b),\pi_\phi)]|\nonumber\\
        \leq&\ \frac{R_{\rm max}}{1-\gamma}\varepsilon+L_Q\varepsilon
    \end{align}
\end{proof}

\paragraph{Proof of Theorem~\ref{Thm:db_sp_Lphi}.}
\begin{proof}
    Using Theorem~\ref{Thm:8} and the definition of $L_\phi^{[2]}$ in Theorem~\ref{Thm:meta}, we have $L_\phi^{[2]}=\frac{R_{\rm max}}{1-\gamma}+L_Q$%, and $L_\CAL{E}=\frac{8R_{\rm max}}{1-\gamma}\cdot\big((1+\gamma)L_Q+\frac{R_{\rm max}}{1-\gamma}\big)$.
\end{proof}

\paragraph{Proof of Theorem~\ref{thm:final}.}
\begin{proof}
    Combining Theorem~\ref{Thm:6.}, Theorem~\ref{Thm:db_sp_Lphi} and the Meta-theorem~\ref{Thm:meta}, we have $|J_{\hat{Q}^\pi}(\pi)-J(\pi)|\leq\frac{\sqrt{C_\pi(\varepsilon)}}{1-\gamma}\cdot
        \sqrt{\sqrt{\frac{32R_{\rm max}^4}{n(1-\gamma)^4}\cdot\log\frac{2|\mathcal{F}|}{\delta}}+L_{\CAL{E}}\varepsilon}+L_\phi\varepsilon$, where
    $L_\CAL{E}=\frac{8R_{\rm max}}{1-\gamma}\cdot\big((1+\gamma)L_Q+\frac{R_{\rm max}}{1-\gamma}\big)$.
    Note that when applying the Meta-theorem~\ref{Thm:meta}, we specify $\mathfrak{est}^\phi(\hat{Q}^\pi)=J_{\hat{Q}^\pi}(\pi_\phi)$ and $\mathfrak{est}^\phi(Q_\phi^\pi)=J(\pi_\phi)$.
    
    Then, applying the result of Lemma~\ref{lemma:meta} proves the result.
\end{proof}

\paragraph{Proof of Corollary~\ref{coro:finite_sample_1}.}
\begin{proof}
 Notice that the $\varepsilon$ inside the square root always dominates the $\varepsilon$ outside with $\varepsilon$ small enough, therefore, we prove the corollary by
    substituting $L_\CAL{E}\varepsilon$ with $\sqrt{\frac{32R_{\rm max}^4}{n(1-\gamma)^4}\cdot\log\frac{2|\mathcal{F}|}{\delta}}$, and then find the condition that the out side $L_\phi\varepsilon$ can be dominated by the term inside the square root. Such condition can be presented as $\big(\frac{32R_{\rm max}^4}{n(1-\gamma)^4}\cdot\log\frac{2|\mathcal{F}|}{\delta}\big)^{\frac{1}{4}}\leq \frac{L_{\CAL{E}}}{L_\phi}\cdot \frac{\sqrt{2C_\pi^n}}{1-\gamma}$.
    Noticing that the coverage term is generally increasing, and is always bounded below by $1$, we therefore provide a sufficient condition as
    $\big(\frac{32R_{\rm max}^4}{n(1-\gamma)^4}\cdot\log\frac{2|\mathcal{F}|}{\delta}\big)^{\frac{1}{4}}\leq \frac{L_{\CAL{E}}}{L_\phi}\cdot \frac{\sqrt{2}}{1-\gamma}$. Solving it gives us the condition $n\geq 8R_{\rm max}^4(L_\phi/L_\CAL{E})^4\log(2|\CAL{F}|/\delta)$, under which $|J_{\hat{Q}^\pi}(\pi)-J(\pi)|\leq\frac{2\sqrt{C_\pi^n}}{1-\gamma}\cdot
        \big(\frac{128R_{\rm max}^4}{n(1-\gamma)^4}\cdot\log\frac{2|\mathcal{F}|}{\delta}\big)^{\frac{1}{4}}$.
\end{proof}

\section{Future-Dependent Value Functions}\label{appendix:FDVF_figure} 
\begin{figure}[ht]
\centering
\begin{tikzpicture}[node distance=3.5cm, thick, >=Stealth, scale=0.85, every node/.style={transform shape}]

    % Nodes with rounded corners
    \node (true_system) [rectangle, rounded corners=15pt, draw, minimum width=3.5cm, minimum height=1.5cm, align=center] {POMDP};
    \node (binning_system) [rectangle, rounded corners=15pt, draw, minimum width=3.5cm, minimum height=1.5cm, align=center, below of=true_system] {Short-term \\memory POMDP};
    \node (true_estimate) [rectangle, rounded corners=15pt, draw, minimum width=3.5cm, minimum height=1.5cm, align=center, right of=true_system, xshift=7cm] {estimate for\\POMDP };
    \node (binning_estimate) [rectangle, rounded corners=15pt, draw, minimum width=3.5cm, minimum height=1.5cm, align=center, right of=binning_system, xshift=7cm] {estimate for\\Short-term\\memory POMDP};

    % Arrows and labels
    \draw[->, red, thick] (true_system) -- node[above, black] {Curse of Memory \textbf{\color{red}\ding{55}}} (true_estimate);
    \draw[->, darkgreen, thick] (true_system) -- node[left, black] {\shortstack{\circletext{1} state abstraction\\ \ \ \ \  truncation mapping}} (binning_system);
    \draw[->, darkgreen, thick] (binning_system) -- node[below, black] {\circletext{2} finite memory \textbf{\color{darkgreen}\ding{51}}} (binning_estimate);
    \draw[->, darkgreen, thick] (binning_estimate) -- node[right, black] {\shortstack{\circletext{3} true algorithm\\v.s.\\ abstract agorithm}} (true_estimate);

    % Policy labels
    \node at ([yshift=0.3cm]true_system.north) {policy: $\pi$};
    \node at ([yshift=-0.3cm]binning_system.south) {policy: $\pi_\Phi$};
    \node at ([yshift=0.3cm]true_estimate.north) {policy: $\pi$};
    \node at ([yshift=-0.3cm]binning_estimate.south) {policy: $\pi_\Phi$};

\end{tikzpicture}
\caption{FDVF analysis pipeline}
\label{fig:FDVF_system_abstraction}
\end{figure}
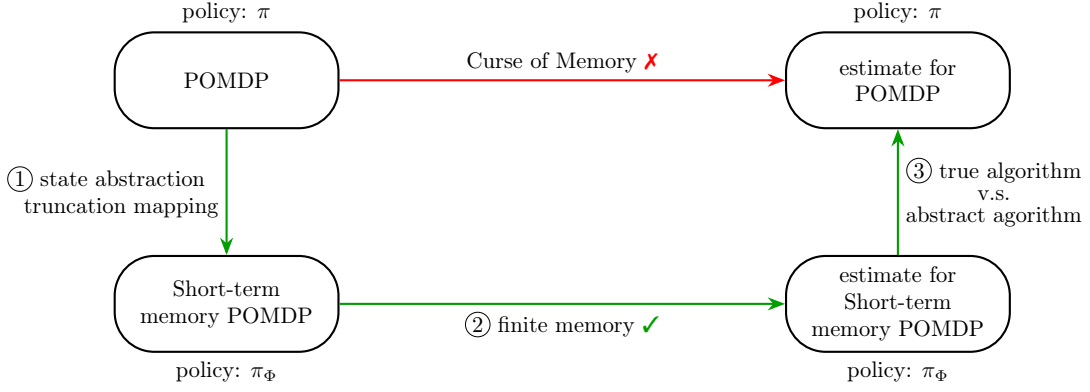

\begin{definition}
MDPs $M_1 = \{\CAL{S}_1, \CAL{A}, P_1, R_1, H\}$ and $M_2 = \{\CAL{S}_2, \CAL{A}, P_2, R_2, H\}$ are said to be isomorphic if there exists a bijection $\varphi: \CAL{S}_1 \to \CAL{S}_2$ such that
$
\varphi(M_1) := \{\varphi(\CAL{S}_1), \CAL{A},P_1(\varphi(\cdot), \cdot), R_1(\varphi(\cdot), \cdot), H\} = M_2
$
\end{definition}

\begin{theorem}\label{Thm:POMDP_existance}
For any POMDP $\mathcal{P}$ and $T \in \N^+$, there exists a short-term memory POMDP $\mathcal{P}_T$ with memory window $T$ such that the belief MDP of $\mathcal{P}$, after abstraction by $\phi_T$, is isomorphic to the belief MDP of $\mathcal{P}_T$. %Specifically for this section, we construct $\mathcal{P}_T$ such that it shares the exact same first $T$ action-observation dynamic with $\mathcal{P}$.
\end{theorem}
\begin{proof}
To prove the existence, it suffices to construct a short-term memory POMDP $\mathcal{P}_T$. Consider the belief MDP $\CAL{M}$ of the original POMDP $\mathcal{P}$, and let $\CAL{M}_T$ be the abstraction of $\CAL{M}$ through $\phi_T$. Let $\mathbf{b}$ denote the belief mapping in the original POMDP. Now, construct $\mathcal{P}_T = (\CAL{S}', \CAL{A}, \CAL{O}, r', H, \mathbb{O}', \mathbb{T}')$ as follows:

Define $\CAL{S}' := \bigcup_{i=0}^{T-1} (\CAL{O} \times \CAL{A})^i \times \CAL{O}$, and the observation function as $\mathbb{O}'(o | s' = (o_1, a_1, \cdots, o_T)) := \mathbb{I}\{o = o_T\}$, which is a one-hot vector.
The reward function is defined as $r'(s' = (o_1, a_1, \cdots, o_T), a) := R_\phi(\phi_T(\mathbf{b}(s')), a)$, and the transition probability as $\mathbb{T}(s'_1 | s'_0, a) := P_\phi(\phi_T(\mathbf{b}(s'_0)) | \phi_T(\mathbf{b}(s'_1)), a)$, where $P_\phi$ is defined in~\eqref{equ:3-26}.

Next, we verify that the belief MDP $\CAL{M}_T$ of POMDP $\mathcal{P}_T$ is indeed isomorphic to the abstraction of the belief MDP $\CAL{M}$ of $\mathcal{P}$. Notice that for every historical sequence in $\CAL{P}_T$, its state can be uniquely determined simply by taking the last $T$ elements of the sequence. That is, the belief states in $\CAL{P}_T$ are one-hot encoded.
Thus, the belief MDP of $\CAL{P}_T$ is isomorphic to the hidden underlying MDP of $\CAL{P}_T$. According to the definitions above, this underlying MDP is naturally isomorphic to $\CAL{M}_T$.

This completes the construction and the proof.
\end{proof}

\begin{comment}
With this short-term memory POMDP (whose memory is limited to the time window $T$), we can now run the abstract algorithm on it. Naturally, the policy to be evaluated must also be abstracted (i.e., truncated), transforming from the original $\pi$ into $\pi_\phi$. To ensure that the truncation does not introduce excessive error, we must also invoke the fast-forgetting assumption on the policy, which is weaker than the Lipchitz Assumption~\ref{asmp:lip_policy} as shown in Lemma~\ref{lem:weaker}.

\begin{assumption}[Fast-Forgetting Policy]\label{asmp:fast_forgetting_policy}
    For the abstraction mapping $\phi_T$, it holds that for all $\varepsilon > 0$, there exists a $T \in \N^+$, such that for all $\tau_h^{[1]+}, \tau_h^{[2]+} \in \CAL{H}^+$ and all $\pi \in {\pi_e, \pi_b}$, if $\tilde{\phi}_T(\tau_h^{[1]+}) = \tilde{\phi}_T(\tau_h^{[2]+})$, then $\|\pi(\tau_h^{[1]+}) - \pi(\tau_h^{[2]+})\|_1 \leq L_\pi \varepsilon$. We denote the dependency of $T$ on $\varepsilon$ as $T_1(\varepsilon)$.
\end{assumption}
\begin{lemma}[Fast-Forgetting weaker than Lipchitz]\label{lem:weaker}
If Assumption~\ref{asmp:fast_forgetting_POMDP} and Assumption~\ref{asmp:lip_policy} hold, then Assumption~\ref{asmp:fast_forgetting_policy} holds aotomatically, with $T_1=T_0$.
\end{lemma}
\end{comment}

\begin{remark}
As discussed above and in the main text, because this short-term memory POMDP is induced by an abstraction mapping $\phi$, and this abstraction mapping $\phi$ guarantees that all belief states mapped to the same representative are close to each other (Assumption~\ref{asmp:fast_forgetting_POMDP}), we can directly apply the conclusions from Theorem~\ref{Thm:L_phi_1} for abstraction error control.

Note that policy truncation is essential here. This is not only to directly reuse the conclusions from Theorem~\ref{Thm:L_phi_1}, but also due to the “curse of memory”—the memory of a policy can severely affect the quality of theoretical guarantees.
\end{remark}

\subsection{Real Algorithm vs. Abstract Algorithm}
\begin{comment}
\begin{definition}
    We define $\|\CAL{V}\|_\infty:=\max_{V\in\CAL{V}}\|V\|_\infty$ (similar for $\Theta$), $C_{\CAL{V}}:=\max\{\|\CAL{V}\|_\infty+1,\|\Theta\|_\infty\}$, $C_\mu := \max_h \max_{a_h, \tau_h^+} \mu(a_h, \tau_h^+)$, and
$
L_{\CAL{E}}:=3\bigg(\frac{c_1H{(C_\mu+1)L_\pi\|\CAL{V}\|_\infty\|\Theta\|_\infty}}{\min_h\min_{a_h,\tau_h^+}\pi_b(a_h|\tau_h^+)}+C_\mu\|\CAL{V}\|_\infty\|\Theta\|_\infty + \frac{c_2 H \max\{C_\mu\|\CAL{V}\|_\infty\|\Theta\|_\infty, \frac{1}{2}\|\Theta\|_\infty^2\}}{\min\{\min_h\min_{o_h,\tau_h}P(o_h|\tau_h), \min_h\min_{a_h,\tau_h^+}\pi_b(a_h|\tau_h^+)/L_\pi\}}\bigg).
$
\end{definition}
First, we consider these two conditions that limits $\varepsilon$ from being too large:
\begin{condition}\label{cond:0}
    The $\varepsilon$ is small enough that $L_\pi\varepsilon/\min_h\min_{a_h,\tau_h^+}\pi_b(a_h|\tau_h^+)\leq \frac{1}{2}$.
\end{condition}
\begin{condition}\label{cond:1}
    The $\varepsilon$ is small enough that $\frac{H \varepsilon}{\min\{\min_h\min_{o_h,\tau_h}P(o_h|\tau_h), \min_h\min_{a_h,\tau_h^+}\pi_b(a_h|\tau_h^+)/L_\pi\}}\leq 1$.
\end{condition}
\end{comment}
The differences between the real algorithm and the abstract algorithm come from three aspects:

1. The discrepancy between $\mu(a_h,\tau_h^+) = \pi_e(a_h|\tau_h^+) / \pi_b(a_h|\tau_h^+)$ and the truncated version $\mu(a_h,\tau_{[h-T+1:h]}^+) = \pi_e(a_h|\tau_{[h-T+1:h]}^+) / \pi_b(a_h|\tau_{[h-T+1:h]}^+)$.
This discrepancy can be controlled by the following lemma:

\begin{lemma}\label{lem:error_1}
If Assumption~\ref{asmp:fast_forgetting_policy} hold, then for any $\varepsilon > 0$ that satisfies condition~\ref{cond:0}, with $T \geq T_1(\varepsilon)$, we have,
\begin{align}
|\mu(a_h,\tau_h^+)-\mu(a_h,\tau_{[h-T+1:h]}^+)|\leq \frac{2{(C_\mu+1)}L_\pi \varepsilon}{\min_h\min_{a_h,\tau_h^+}\pi_b(a_h|\tau_h^+)}
\end{align}
\end{lemma}
\begin{proof} Using Condition~\ref{cond:0}, we have
\begin{align}
    |\mu(a_h,\tau_h^+)-\mu(a_h,\tau_{[h-T+1:h]}^+)|\leq&\  \frac{|\pi_e(a_h|\tau_h^+)-\pi_e(a_h|\tau_{[h-T+1:h]}^+)|}{\pi_b(a_h|\tau_h^+)}\nonumber\\&\quad+\pi_e(a_h|\tau_{[h-T+1:h]}^+)\bigg|\frac{1}{\pi_b(a_h|\tau_h^+)}-\frac{1}{\pi_b(a_h|\tau_{[h-T+1:h]}^+)}\bigg|\\
    \leq&\ \frac{L_\pi\varepsilon}{\min_h\min_{a_h,\tau_h^+}\pi_b(a_h|\tau_h^+)}\bigg(1+\frac{\pi_e(a_h|\tau_{[h-T+1:h]}^+)}{\pi_b(a_h|\tau_{[h-T+1:h]}^+)}\bigg)\\
    \leq&\ \frac{L_\pi\varepsilon}{\min_h\min_{a_h,\tau_h^+}\pi_b(a_h|\tau_h^+)}(2+2C_\mu)
\end{align}
which proves the lemma.
\end{proof}

2. The discrepancy between $V(f_h',\tau_h')$ and $V(f_h',\tau_{[h-T+1:h]}')$.
This requires the function class to forget historical information quickly, as stated below:

\begin{assumption}[Fast-Forgotten Function Class]\label{asmp:fast_forgetting_function_class}
Consider the function class used for estimation $\CAL{V} : \CAL{F} = (\CAL{F}' \times \CAL{H}) \to \mathbb{R}$. It satisfies that for all $\varepsilon > 0$, there exists $T \in \N^+$ such that for all $V \in \CAL{V}$,
\begin{align}
|V(f_h,\tau_h)-V(f_{[h:h+T]},\tau_{[h-T+1:h]})|\leq \|\CAL{V}\|_\infty\varepsilon
\end{align}
The suitable values of $T$ form a function of $\varepsilon$, denoted as $T_2(\varepsilon)$.
\end{assumption}
Note that the essential assumption here is that the "history" in the extended future is fast-forgetting. Since in the original literature of FDVF~\cite{NEURIPS2023_3380e811}, the future is by default truncated by a length $M_F$.
%Note that this assumption imposes no constraints on the “future” component—only the history is required to exhibit a fast-forgetting property.

3. The difference in data-generating distribution between the real POMDP and the abstract short-term memory POMDP.
This discrepancy arises from two sources, firstly that the behavior policy is truncated, and secondly, the transition probabilities of the POMDP differ slightly.

Let $w^{\phi}(f_1)$ denote the importance weight accounting for this distribution shift. Then we define:
\begin{align}
w^{\phi_T}(f_1):=\frac{\pi_b^{\phi_T}(a_1|\tau_1^+)}{\pi_b(a_1|\tau_1^+)}\cdot\frac{P^{\phi_T}(o_2|\tau_2)}{P(o_2|\tau_2)}\cdot\cdots\cdot\frac{\pi_b^{\phi_T}(a_H|\tau_H^+)}{\pi_b(a_H|\tau_H^+)}
\end{align}

Under the assumptions that both the POMDP and the policy are fast-forgetting, we have the following lemma:

\begin{lemma}\label{lem:importance_weight}
If Assumptions~\ref{asmp:fast_forgetting_POMDP} and~\ref{asmp:fast_forgetting_policy} hold, then for any $\varepsilon > 0$, with $T \geq \max\{T_0(\varepsilon),T_1(\varepsilon)\}$, then for $\varepsilon$ small enough, namely, when Condition~\ref{cond:1} is satisfied, we have
\begin{align}
|w^{\phi_T}(f_1)|\leq 1 + \frac{3 H \varepsilon}{\min\{\min_h\min_{o_h,\tau_h}P(o_h|\tau_h), \min_h\min_{a_h,\tau_h^+}\pi_b(a_h|\tau_h^+)/L_\pi\}}
\end{align}
%The values of $T$ satisfying this condition form a function of $\varepsilon$, denoted $T_3(\varepsilon)$.
\end{lemma}
\begin{proof}
    We first show that using Condition~\ref{cond:1}, for any $h$, $\big|1-\frac{\pi_b^{\phi_T}(a_h|\tau_h^+)}{\pi_b(a_h|\tau_h^+)}\big|\leq \frac{L_\pi\varepsilon}{\min_h\min_{a_h,\tau_h^+}\pi_b(a_h|\tau_h^+)}$.

    Then using the fact that
    \begin{align}
        P^{\phi_T}(o_h|\tau_h)=P^{\phi_T}(o_h|\mathbf{b}(\tau_{h-1}^+),a_{h-1})=\mathbb{E}_{b\sim p_{\phi_T(\mathbf{b}(\tau_{h-1}^+))}}[P(o_h|b,a_{h-1})]
    \end{align}
    and that for all $b\in\phi_T^{-1}(\phi_T(\mathbf{b}(\tau_{h-1}^+)))$, $\|b-\mathbf{b}(\tau_{h-1}^+)\|_1\leq \varepsilon$, which is described in Assumption~\ref{asmp:fast_forgetting_POMDP}. It then follows that
    \begin{align}
        |P(o_h|\tau_h)-P^{\phi_T}(o_h|\tau_h)|=&\ |P(o_h|\mathbf{b}(\tau_{h-1}^+),a_{h-1})-\mathbb{E}_{b\sim p_{\phi_T(\mathbf{b}(\tau_{h-1}^+))}}[P(o_h|b,a_{h-1})]|\nonumber\\
        =&\ |\mathbb{E}_{b\sim p_{\phi_T(\mathbf{b}(\tau_{h-1}^+))}}[P(o_h|\mathbf{b}(\tau_{h-1}^+),a_{h-1})-P(o_h|b,a_{h-1})]|\nonumber\\
        \leq&\ \mathbb{E}_{b\sim p_{\phi_T(\mathbf{b}(\tau_{h-1}^+))}}[|P(o_h|\mathbf{b}(\tau_{h-1}^+),a_{h-1})-P(o_h|b,a_{h-1})|]\\
        \leq&\ \varepsilon
    \end{align}
    where the last step uses Lemma~\ref{lemma:3}. Now we have
    $\big|1-\frac{P^{\phi_T}(o_h|\tau_h)}{P(o_h|\tau_h)}\big|\leq \frac{\varepsilon}{\min_h\min_{o_h,\tau_h}P(o_h|\tau_h)}$.

    Therefore, $w^{\phi_T}\leq (1+\frac{\varepsilon}{\min\{\min_h\min_{o_h,\tau_h}P(o_h|\tau_h), \min_h\min_{a_h,\tau_h^+}\pi_b(a_h|\tau_h^+)/L_\pi\}})^H$, combining Condition~\ref{cond:1} we have $w^{\phi_T}\leq 1 + \frac{3 H \varepsilon}{\min\{\min_h\min_{o_h,\tau_h}P(o_h|\tau_h), \min_h\min_{a_h,\tau_h^+}\pi_b(a_h|\tau_h^+)/L_\pi\}}$.
\end{proof}

To summarize, by considering all sources of error, we have the following theorem.
\begin{lemma}\label{lemma:11_fdvf}
    Define
\begin{align}
    \CAL{E}_{\CAL{V},\Theta}(V,\theta):=&\ \sum_{h=1}^H\mathbb{E}_{\CAL{D}}[\{\mu(a_h,\tau_h^+)(r_h+V(f_{h+1}))- V(f_h)\}\theta(\tau_h)-\frac{1}{2}\theta(\tau_h)^2]\\
        \CAL{E}^{\phi_T}_{\CAL{V},\Theta}(V,\theta):=&\ \sum_{h=1}^H\mathbb{E}_{\CAL{D}}[w^{\phi_T}(\{\mu(a_h,\tau_{h-T+1:h}^+)(r_h+V(\phi_T(f_{h+1})))-V(\phi_T(f_h))\}\theta(\tau_h)\nonumber\\&\qquad\qquad\qquad\qquad\qquad\qquad\qquad\qquad\qquad\qquad\qquad\qquad\qquad-\frac{1}{2}\theta(\tau_h)^2)]
\end{align}
where $\phi_T(f_{h})=\phi_T(f'_h,\tau_h)=(f'_h,\phi_T(\tau_h))$.
If Assumptions~\ref{asmp:fast_forgetting_POMDP},\ref{asmp:fast_forgetting_policy}, and\ref{asmp:fast_forgetting_function_class} all hold, then for any $\varepsilon > 0$ satisfying condition~\ref{cond:0},~\ref{cond:1} and for any $V \in \CAL{V},\theta\in\Theta$, we have:
\begin{align}
|\CAL{E}_{\CAL{V},\Theta}(V,\theta)-\CAL{E}^{\phi_T}_{\CAL{V},\Theta}(V,\theta)|\leq&\ \frac{1}{3}L_{\CAL{E}}\varepsilon
\end{align}

and $T=\max\{T_0(\varepsilon),T_1(\varepsilon),T_2(\varepsilon)\}$.
\end{lemma}
\begin{proof}
    The proof of such lemma uses triangle's inequality, by summing up the error depicted in Lemma~\ref{lem:error_1}, Assumption~\ref{asmp:fast_forgetting_function_class}, and Lemma~\ref{lem:importance_weight}, we get the result.
\end{proof}
\begin{theorem}\label{Thm:5-20}
Define
\begin{align}
    \CAL{E}_{\CAL{V}}(V):=&\ \max_{\theta \in \Theta}\CAL{E}_{\CAL{V},\Theta}(V,\theta),\quad
        \CAL{E}^{\phi_T}_{\CAL{V}}(V):=\max_{\theta \in \Theta}\CAL{E}_{\CAL{V},\Theta}^\phi(V,\theta)
\end{align}

If Assumptions~\ref{asmp:fast_forgetting_POMDP},\ref{asmp:fast_forgetting_policy}, and\ref{asmp:fast_forgetting_function_class} all hold, then for any $\varepsilon > 0$ satisfying condition~\ref{cond:0},~\ref{cond:1} and for any $V \in \CAL{V}$, we have:
\begin{align}
|\CAL{E}_{\CAL{V}}(V)-\CAL{E}^{\phi_T}_{\CAL{V}}(V)|\leq&\ \frac{2}{3}L_{\CAL{E}}\varepsilon
\end{align}

and $T=\max\{T_0(\varepsilon),T_1(\varepsilon),T_2(\varepsilon)\}$.
\end{theorem}
\begin{proof}
    This use the observation that if $\forall\theta, |\CAL{E}_{\CAL{V},\Theta}(V,\theta)-\CAL{E}_{\CAL{V},\Theta}^{\phi_T}(V,\theta)|\leq \frac{1}{3}L_\CAL{E}\varepsilon$, then $|\sup_{\theta}\CAL{E}_{\CAL{V},\Theta}(V,\theta)-\sup_{\theta}\CAL{E}_{\CAL{V},\Theta}^{\phi_T}(V,\theta)|\leq 2\cdot\frac{1}{3}L_\CAL{E}\varepsilon$.
\end{proof}

\paragraph{Proof of Theorem~\ref{Thm:def_Lphi_fdvf}.}
\begin{proof}
    To prove the theorem, we need to prove $L_\phi^{[2]}:=\|\CAL{V}\|_\infty$ satisfies $|\mathbb{E}_{\pi_b}[\hat{V}(f_1)]-\mathbb{E}_{\pi_b^\phi}[\hat{V}(f_1)]|\leq L_\phi^{[2]}\varepsilon$.

    Recall that in our construction, the dynamic of the short-term memory POMDP ensures that the first $T$ action-observation has the exact same dynamic as the true POMDP. Also notice that $\forall f_1^{[1]},f_1^{[2]}\in\CAL{F}_1$, if $f_1^{[1]},f_1^{[2]}$ shares the first $T$ pairs of action and observation, then $|\hat{V}(f_1^{[1]})-\hat{V}(f_1^{[2]})|\leq \|\CAL{V}\|_\infty\varepsilon$ as indicated by the property of the function class $\CAL{V}$, Assumption~\ref{asmp:fast_forgetting_function_class}. Then
    \begin{align}
        |\mathbb{E}_{\pi_b}[\hat{V}(f_1)]-\mathbb{E}_{\pi_b^\phi}[\hat{V}(f_1)]|\leq&\ \mathbb{E}_{f_{1:T}\sim\pi_b}[|\mathbb{E}_{f_{T+1}\sim\pi_b}[\hat{V}(f_1)|f_{1:T}]-\mathbb{E}_{f_{T+1}\sim\pi_b^\phi}[\hat{V}(f_1)|f_{1:T}]|]\nonumber\\\leq&\  \|\CAL{V}\|_\infty\varepsilon
    \end{align}
    which proves the theorem.
\end{proof}

\subsection{Theoretical Guarantee of FDVF}
\paragraph{Proof of Theorem~\ref{Thm:finite_sample_FDVF}.}
\begin{proof}
    Let $\hat{V}:=\argmin_{V\in\CAL{V}}\CAL{E}_{\CAL{V}}(V)$, and correspondingly $\hat{V}_{\phi}:=\argmin_{V\in\CAL{V}}\CAL{E}_{\CAL{V}}^{\phi_T}(V)$. Our first goal is to show that with probability greater than $1-\delta$,
    \begin{align}
        |J^\phi(\pi_e^\phi)-\mathbb{E}_{\pi_b^\phi}[\hat{V}(f_1)]|\leq&\  \sqrt{H} \cdot \max_{h\in[H]}\sup_{V\in\CAL{V}}\sqrt{\frac{\mathbb{E}_{\pi_e^\phi}[(\CAL{B}^{(\CAL{S},\CAL{H}_T)}V)(s_h,\tau_{[h-T+1:h]})^2]}{\mathbb{E}_{\pi_b^\phi}[(\CAL{B^H}V)(\tau_h)^2]}}\nonumber\\&\qquad\cdot\sqrt{\frac{cHC_{\CAL{V}}^2C_\mu}{n}\log\frac{|\CAL{V}||{\Theta}|}{\delta}+L_{\CAL{E}}\varepsilon} \label{inequ:fdvf_proof_-1}
    \end{align}
To do this, we follow the proof provided by~\cite{zhang2024cursesfuturehistoryfuturedependent}, define
$X_{V,h}^{\phi}:=\mu(a_h,\tau_h^+)(r_h+V(\phi_T(f_{h+1})))-V(\phi_T({f_h}))$, $X_{V,h}:=\mu(a_h,\tau_h^+)(r_h+V(f_{h+1}))-V(f_h)$, then $\CAL{E}_{\CAL{V}}^{\phi_T}(V)=\frac{1}{2}\max_{\theta\in\Theta}\sum_{h=1}^H\mathbb{E}_{\CAL{D}}[w^{\phi_T}((X_{V,h}^\phi)^2-(X_{V,h}^\phi-\theta(\tau_h))^2)]$, and such $\theta$ that achieves maximum is denoted as $\hat{\theta}_{V}$. Similarly, $\CAL{E}_{\CAL{V}}(V)=\frac{1}{2}\max_{\theta\in\Theta}\sum_{h=1}^H\mathbb{E}_{\CAL{D}}[X_{V,h}^2-(X_{V,h}-\theta(\tau_h))^2)$ and the $\theta$ that achieves maximum is represented by $\hat{\theta}^{0}_V$. According to a concentration argument using Bernstein’s inequality as presented in first part (\textit{Analysis of Inner Maximizer}) of the the proof of theorem 2 in~\cite{zhang2024cursesfuturehistoryfuturedependent}, we arrive at an argument that indicates with probability greater than $1-\delta/2$, for any $V\in\CAL{V},\theta\in\Theta$,
\begin{align}
    \bigg|
    \sum_{h=1}^H\mathbb{E}_{\CAL{D}}[w^{\phi_T}(\hat{\theta}_V(\tau_h)-X_{V,h}^\phi)^2]-&\ \sum_{h=1}^H\mathbb{E}_{\CAL{D}}[w^{\phi_T}(X_{V,h}^\phi-(\CAL{B^H}V)(\tau_h))^2]
    \bigg|\nonumber\\&\qquad\leq \frac{675HC_{\CAL{V}}^2C_\mu\|w^{\phi_T}\|_\infty}{n}\cdot\log\frac{4|\CAL{V}||\Theta|}{\delta}=:\eta\label{inequ:fdvf_proof_0}
\end{align}
Then we have
\begin{align}
    &\ \sum_{h=1}^H\mathbb{E}_{\CAL{D}}[w^{\phi_T}((X_{\hat{V},h}^\phi)^2-(X_{\hat{V},h}^\phi-(\CAL{B^H}\hat{V})(\tau_h))^2)]\nonumber\\\leq&\  \sum_{h=1}^H\mathbb{E}_{\CAL{D}}[w^{\phi_T}((X_{\hat{V},h}^\phi)^2-(\hat{\theta}_{\hat{V}}(\tau_h)-X_{\hat{V},h}^\phi)^2)]+\eta\label{inequ:fdvf_proof_1}\\
    \leq&\ \sum_{h=1}^H\mathbb{E}_{\CAL{D}}[X_{\hat{V},h}^2-(\hat{\theta}_{\hat{V}}(\tau_h)-X_{\hat{V},h})^2]+\frac{1}{3}L_\CAL{E}\varepsilon+\eta\label{inequ:fdvf_proof_2}\\
    \leq&\ \sum_{h=1}^H\mathbb{E}_{\CAL{D}}[X_{\hat{V},h}^2-(\hat{\theta}_{\hat{V}}^0(\tau_h)-X_{\hat{V},h})^2]+\frac{1}{3}L_\CAL{E}\varepsilon+\eta\label{inequ:fdvf_proof_3}\\
    \leq&\ \sum_{h=1}^H\mathbb{E}_{\CAL{D}}[w^{\phi_T}((X_{\hat{V}_\phi,h}^\phi)^2-(\hat{\theta}_{\hat{V}_\phi}(\tau_h)-X_{\hat{V}_\phi,h}^\phi)^2)]+L_\CAL{E}\varepsilon+\eta\label{inequ:fdvf_proof_4}
\end{align}
Here,~\eqref{inequ:fdvf_proof_1} uses~\eqref{inequ:fdvf_proof_0}. \eqref{inequ:fdvf_proof_2} uses Lemma~\ref{lemma:11_fdvf} and the fact that $\CAL{E}_{\CAL{V},\Theta}(\hat{V},\hat{\theta}_{\hat{V}})=\sum_{h=1}^H\mathbb{E}_{\CAL{D}}[X_{\hat{V},h}^2-(\hat{\theta}_{\hat{V}}(\tau_h)-X_{\hat{V},h})^2]$ and that $\CAL{E}_{\CAL{V},\Theta}^\phi(\hat{V},\hat{\theta}_{\hat{V}})=\sum_{h=1}^H\mathbb{E}_{\CAL{D}}[w^{\phi_T}((X_{\hat{V},h}^\phi)^2-(\hat{\theta}_{\hat{V}}(\tau_h)-X_{\hat{V},h}^\phi)^2)]$. After that, ~\eqref{inequ:fdvf_proof_3} uses the fact that $\hat{\theta}_{\hat{V}}^0=\argmax_{\theta\in\Theta}\CAL{E}_{\CAL{V},\Theta}(\hat{V},\theta)$, thus $\CAL{E}_{\CAL{V},\Theta}(\hat{V},\hat{\theta}_{\hat{V}}^0)\geq \CAL{E}_{\CAL{V},\Theta}(\hat{V},\hat{\theta}_{\hat{V}})$. \eqref{inequ:fdvf_proof_4} uses Theorem~\ref{Thm:5-20} and the fact that $\min_{V\in\CAL{V}}\CAL{E}_{\CAL{V}}(V)=\sum_{h=1}^H\mathbb{E}_{\CAL{D}}[X_{\hat{V},h}^2-(\hat{\theta}_{\hat{V}}^0(\tau_h)-X_{\hat{V},h})^2]$ and $\min_{V\in\CAL{V}}\CAL{E}_{\CAL{V}}^\phi(V)=\sum_{h=1}^H\mathbb{E}_{\CAL{D}}[w^{\phi_T}((X_{\hat{V}_\phi,h}^\phi)^2-(\hat{\theta}_{\hat{V}_\phi}(\tau_h)-X_{\hat{V}_\phi,h}^\phi)^2)]$. Noticing that $\min_{V\in\CAL{V}}\CAL{E}_{\CAL{V}}\leq \min_{V\in\CAL{V}}\CAL{E}_{\CAL{V}}^\phi(V)+\sup_{V\in\CAL{V}}|\CAL{E}_{\CAL{V}}(V)-\CAL{E}_{\CAL{V}}^\phi(V)|$ finish the derivation of~\eqref{inequ:fdvf_proof_4}.

After that, notice the abstract realizability assumption $\exists V_{\CAL{F}}^\phi\in\CAL{V}$, and that for any $V_{\CAL{F}}^\phi$, we have
\begin{align}
    &\ \sum_{h=1}^H\mathbb{E}_{\CAL{D}}[w^{\phi_T}((X_{\hat{V}_\phi,h}^\phi)^2-(\hat{\theta}_{\hat{V}_\phi}(\tau_h)-X_{\hat{V}_\phi,h}^\phi)^2)]\nonumber\\\leq&\  \sum_{h=1}^H\mathbb{E}_{\CAL{D}}[w^{\phi_T}((X_{V_{\CAL{F}}^\phi,h}^\phi)^2-(\hat{\theta}_{V_{\CAL{F}}^\phi}(\tau_h)-X_{V_{\CAL{F}}^\phi,h}^\phi)^2)]\\
    \leq&\  \sum_{h=1}^H\mathbb{E}_{\CAL{D}}[w^{\phi_T}((X_{V_{\CAL{F}}^\phi,h}^\phi)^2-(X_{V_{\CAL{F}}^\phi,h}^\phi-(\CAL{B^H}V_{\CAL{F}}^\phi)(\tau_h))^2)]+\eta=\eta\label{equ:fdvf_proof_5}
\end{align}
where the second last inequality uses the minimal property of $\hat{V}_\phi$, and the last inequality uses~\eqref{inequ:fdvf_proof_0}. The last equality is the result of the definition of $V_{\CAL{F}}^\phi$ that it is the zero point of bellman residual operator $\CAL{B^H}$.

Combining~\eqref{equ:fdvf_proof_5} and~\eqref{inequ:fdvf_proof_4}, we have
\begin{align}
    \sum_{h=1}^H\mathbb{E}_{\CAL{D}}[w^{\phi_T}((X_{\hat{V},h}^\phi)^2-(X_{\hat{V},h}^\phi-(\CAL{B^H}\hat{V})(\tau_h))^2)]\leq L_\CAL{E}\varepsilon+\frac{1350HC_{\CAL{V}}^2C_\mu\|w^{\phi_T}\|_\infty}{n}\cdot\log\frac{4|\CAL{V}||\Theta|}{\delta}\label{inequ:fdvf_proof_6}
\end{align}

The next step is identical to the equation (17) in~\cite{zhang2024cursesfuturehistoryfuturedependent}, which, with the help of Bernstein's inequality, gives us that with probability greater than $1-\delta/2$, for any $V\in\CAL{V}$,
\begin{align}
    &\ \bigg|
    \sum_{h=1}^H\{\mathbb{E}_{\CAL{D}}-\mathbb{E}_{\pi_b^\phi}\}[w^{\phi_T}((X_{V,h}^\phi)^2-(X_{V,h}^\phi-(\CAL{B^H}V)(\tau_h))^2)]
    \bigg|\nonumber\\
    \leq&\ \sqrt{\frac{58HC_{\CAL{V}}^2C_\mu\|w^{\phi_T}\|_\infty}{n}\cdot\log\frac{4|\CAL{V}|}{\delta}\cdot \sum_{h=1}^H\mathbb{E}_{\pi_b^\phi}[(\CAL{B^H}V)(\tau_h)^2]}+\frac{27HC_{\CAL{V}}^2C_\mu\|w^{\phi_T}\|_\infty}{n}\cdot\log\frac{4|\CAL{V}|}{\delta}\label{inequ:fdvf_proof_7}.
\end{align}

Therefore, combining~\eqref{inequ:fdvf_proof_6} and~\eqref{inequ:fdvf_proof_7} we get
\begin{align}
    \sum_{h=1}^H\mathbb{E}_{\pi_b^\phi}[(\CAL{B^H}V)(\tau_h)^2]\leq&\  L_\CAL{E}\varepsilon+\frac{1377HC_{\CAL{V}}^2C_\mu\|w^{\phi_T}\|_\infty}{n}\cdot\log\frac{4|\CAL{V}||\Theta|}{\delta}\nonumber\\
    &\ +\sqrt{\frac{58HC_{\CAL{V}}^2C_\mu\|w^{\phi_T}\|_\infty}{n}\cdot\log\frac{4|\CAL{V}||\Theta|}{\delta}\cdot \sum_{h=1}^H\mathbb{E}_{\pi_b^\phi}[(\CAL{B^H}V)(\tau_h)^2]},
\end{align}
solving which gives us the final result that
\begin{align}
    &\ \sum_{h=1}^H\mathbb{E}_{\pi_b^\phi}[(\CAL{B^H}V)(\tau_h)^2]\leq \frac{1406HC_{\CAL{V}}^2C_\mu\|w^{\phi_T}\|_\infty}{n}\cdot\log\frac{4|\CAL{V}||\Theta|}{\delta}+L_\CAL{E}\varepsilon\nonumber\\& + \sqrt{\frac{80707H^2C_{\CAL{V}}^4C_\mu^2\|w^{\phi_T}\|_\infty^2}{n^2}\cdot\bigg(\log\frac{4|\CAL{V}||\Theta|}{\delta}\bigg)^2+\frac{29HC_{\CAL{V}}^2C_\mu\|w^{\phi_T}\|_\infty}{n}\cdot\log\frac{4|\CAL{V}||\Theta|}{\delta}\cdot 2L_\CAL{E}\varepsilon}\\
    &\leq \frac{(1406+\sqrt{80707+29C})HC_{\CAL{V}}^2C_\mu\|w^{\phi_T}\|_\infty}{n}\cdot\log\frac{4|\CAL{V}||\Theta|}{\delta}+L_\CAL{E}\varepsilon
\end{align}
where the last step is subject to $2L_{\CAL{E}}\varepsilon\leq \frac{CHC_{\CAL{V}}^2C_\mu\|w^{\phi_T}\|_\infty}{n}\cdot\log\frac{4|\CAL{V}||\Theta|}{\delta}$. Notice that under condition~\ref{cond:1}, $\|w^{\phi_T}\|_\infty\leq e$, so this requirement is covered by condition~\ref{cond:2}.

Then~\eqref{inequ:fdvf_proof_-1} is shown using the telescoping property of bellman residual operator and the fact that $\|w^{\phi_T}\|_\infty\leq e$ as has been mentioned. 

Now that we've obtained~\eqref{inequ:fdvf_proof_-1}, it suffice to prove the theorem using the result from Theorem~\ref{Thm:def_Lphi_fdvf}, which indicates that
\begin{align}
    |J(\pi_e)-J^\phi(\pi_e^\phi)|+|\mathbb{E}_{\pi_b}[\hat{V}(f_1)]-\mathbb{E}_{\pi_b^\phi}[\hat{V}(f_1)]|\leq L_\phi\varepsilon.
\end{align}
And we prove the theorem by applying Meta-theorem~\ref{Thm:meta}. Note that when applying the Meta-theorem, we specify $\mathfrak{est}^\phi(\hat{V})=\mathbb{E}_{\pi_b^\phi}[\hat{V}(f_1)]$ and $\mathfrak{est}^\phi(V_\phi^\pi)=J^\phi(\pi_e^\phi)$, the latter is the ground truth estimation on the abstract system.
\end{proof}

\begin{corollary}
Under the conditions of the Theorem~\ref{Thm:finite_sample_FDVF}, with probability greater then $1 - \delta$, we have:
\begin{align}
&\ |J(\pi_e)-\mathbb{E}_{\pi_b}[\hat{V}(f_1)]|\leq \inf_{\substack{\varepsilon\geq 0\\D(\varepsilon)}}\Bigg(L_\phi\varepsilon+\sqrt{H} \cdot \max_{h\in[H]}\sup_{V\in\CAL{V}}\sqrt{\frac{\mathbb{E}_{\pi_e^\phi}[(\CAL{B}^{(\CAL{S},\CAL{H}_T)}V)(s_h,\tau_{[h-T+1:h]})^2]}{\mathbb{E}_{\pi_b^\phi}[(\CAL{B^H}V)(\tau_h)^2]}}\nonumber\\&\qquad\qquad\quad\qquad\qquad\qquad\qquad\cdot\sqrt{\frac{cHC_{\CAL{V}}^2C_\mu}{n}\log\frac{|\CAL{V}||{\Theta}|}{\delta}+L_{\CAL{E}}\varepsilon} \Bigg)
\end{align}
where $D(\varepsilon)$ stands for such $\varepsilon$ that satisfies abstract realizability, Bellman completeness and condition~\ref{cond:0},~\ref{cond:1} and~\ref{cond:2}.
\end{corollary}
\begin{proof}
    Combining Lemma~\ref{lemma:meta} and Theorem~\ref{Thm:finite_sample_FDVF}, we get the result.
\end{proof}
\begin{comment}
\begin{corollary}[Boosted finite sample guarantee]\label{coro:finite_sample_2}
    For $n$ large enough with necessary realizability condition, we have a finite sample guarantee of 
    
    $|J(\pi_e)-\mathbb{E}_{\pi_b}[\hat{V}(f_1)]|\leq \sqrt{H} \cdot {\displaystyle\max_{h\in[H]}\sup_{V\in\CAL{V}}}\sqrt{\frac{\mathbb{E}_{\pi_e^\phi}[(\CAL{B}^{(\CAL{S},\CAL{H}_T)}V)(s_h,\tau_{[h-T+1:h]})^2]}{\mathbb{E}_{\pi_b^\phi}[(\CAL{B^H}V)(\tau_h)^2]}}\cdot\sqrt{\frac{cHC_{\CAL{V}}^2C_\mu}{n}\log\frac{|\CAL{V}||{\Theta}|}{\delta}}$
\end{corollary}
\end{comment}
\paragraph{Proof of Corollary~\ref{coro:finite_sample_2}}
\begin{proof}
    This is obtained by choosing $L_\CAL{E}\varepsilon=\frac{c'HC_{\CAL{V}}^2C_\mu}{n}\log\frac{|\CAL{V}||{\Theta}|}{\delta}$ for some constant $c'$ in Theorem~\ref{Thm:finite_sample_FDVF}.
\end{proof}

\subsection{A Simpler Pipeline: Abstracting Only the Policy}
\paragraph{Proof of Theorem~\ref{Thm:tighter_fdvf_guarantee}.}
\begin{proof}
    The new $L_\phi$ is obtained using Theorem~\ref{Thm:policy_abstract_error} combining with Theorem~\ref{Thm:def_Lphi_fdvf}, and the new $L_\CAL{E}$ is analogous to the result of Lemma~\ref{lemma:11_fdvf}. The rest of the proof is exactly identical to that of Theorem~\ref{Thm:finite_sample_FDVF}.
\end{proof}

\section{Why our Coverage is Better}\label{appendix:why_better}
\paragraph{Elaboration on example~\ref{example:covering_num_H_infty}}

In this example, we consider a belief space with a smoothness structure (Section 5.3~\cite{lee2007makes}) %\yplu{you can copy the example here} 
 denoted as follow: 

 $\CAL{B}$ is a bounded subset in a $|\CAL{S}|$-dimensional vector space, assume that every belief $b\in\CAL{B}$ can be represented by $m$ basis vectors through linear combinations, and the magnitudes of both the basis elements and the linear coefficients are bounded above by a constant $C$. Then
 the covering number for our belief space scales as $O((C|\CAL{S}|L_\CAL{E}m)^m\cdot (\frac{32R_{\rm max}^4}{n(1-\gamma)^4}\cdot\log\frac{2|\mathcal{F}|}{\delta})^{\frac{m}{2}})$. We assume the coverage being sublinear polynomial w.r.t. its worst case (i.e. the covering number), specifically to the power of $\frac{1}{2m}$. Then we have a finite sample guarantee of $O(\frac{(C|\CAL{S}|L_\CAL{E}mR_{\rm max}^2)^{\frac{1}{4}}}{(1-\gamma)^{\frac{3}{2}}}\cdot (\frac{1}{n}\log\frac{|\mathcal{F}|}{\delta})^{\frac{1}{8}})$. Note that we only assume sublinear polynomial instead of logarithmic since the latter is too strong, and may directly resolve the exponentiality.

%\paragraph{Elaboration on example~\ref{example:Fast_forgetting_H_infty}}

\paragraph{Proof of Theorem~\ref{Thm:coverage_compare_2}}
\begin{proof}%\zyh{I think this is interesting.}
    We proof the theorem by constructing $d^D(\tau')=\sum_{\{\tau:\tilde{\phi}_T(\tau)=\tau'\}}d^{\pi_b}(\tau)$. Then noticing that $d^{\pi_e^\phi}_\phi(\tau')=\sum_{\{\tau:\tilde{\phi}_T(\tau)=\tau'\}}d^{[\pi_e^\phi]_{\rm true}}(\tau)$ is automatically satisfied by how the abstraction $\phi_T$ is defined. Also, it's not difficult to notice that in the one-hot belief state scenario, $\frac{d^{[\pi_e^\phi]_{\rm true}}(s_h,{\tau_h})}{d^{\pi_b}(s_h,{\tau_h})}=\frac{d^{[\pi_e^\phi]_{\rm true}}({\tau_h})}{d^{\pi_b}({\tau_h})}$, and it's exactly the same for the short-term memory POMDP induced by $\tilde{\phi}_T$ as we constructed. Here, $\tilde{\phi}_T:\CAL{H}\to\CAL{H}_T\subset\CAL{H}$.

    Then, consider two $\sigma$-algebras $\CAL{A}:=\CAL{P}(\CAL{H})$ and $\CAL{D}:=\tilde{\phi}_T^{-1}(\CAL{P}(\CAL{H}_T))$, and it's obvious that $\CAL{D}\subset\CAL{A}$. Define the probability point measure $P^{\pi_e}$ and $P^{\pi_b}$ corresponding to the weight function $d^{[\pi_e^\phi]_{\rm true}}$ and $d^{\pi_b}$ on the $\sigma$-algebras $\CAL{A}$, then the probability measure can also be restricted to the smaller $\sigma$-algebra $\CAL{D}$.
    It is easy to notice that the two terms we try to compare coincides with the $\chi^2$-divergence between $P^{\pi_e}$ and $P^{\pi_b}$, where for the LHS we use the coarser $\sigma$-algebra $\CAL{D}$, and use the finer $\sigma$-algebra $\CAL{A}$ for the RHS.

    Then we use the variational representation of $\chi^2$-divergence to obtain our final result, by noticing that
    \begin{align}
        \chi^2_{\CAL{D}}(P^{\pi_e}\|P^{\pi_b})=&\ \sup_{g\in \mathcal{M}(\CAL{D})}\mathbb{E}_{P^{\pi_e}}[g(\tau_h)]-\mathbb{E}_{P^{\pi_b}}[g(\tau_h)^2/4+g(\tau_h)]\\
        \chi^2_{\CAL{A}}(P^{\pi_e}\|P^{\pi_b})=&\ \sup_{g\in \mathcal{M}(\CAL{A})}\mathbb{E}_{P^{\pi_e}}[g(\tau_h)]-\mathbb{E}_{P^{\pi_b}}[g(\tau_h)^2/4+g(\tau_h)]
    \end{align}
    Since $\CAL{D}\subset\CAL{A}$, any $g$ that is $\CAL{D}$ measurable is also $\CAL{A}$ measurable, consequently
    \begin{align}
        \chi^2_{\CAL{D}}(P^{\pi_e}\|P^{\pi_b})\leq \chi^2_{\CAL{A}}(P^{\pi_e}\|P^{\pi_b})
    \end{align}
    which proves the theorem.
\end{proof}

\paragraph{Proof of Theorem~\ref{Thm:coverage_compare_3}}
\begin{proof}
    Construct $d^D$ exactly as in Theorem~\ref{Thm:coverage_compare_2}, then let $w^\star(\tau_h)=\frac{d^{[\pi_e^\phi]_{\rm true}}(\tau_h)}{d^D(\tau_h)}$, and $\tau_h^\star$ is when achieves the maximum. Similarly, let $w^*(\tilde\phi_T(\tau))=\frac{d^{\pi_e^\phi}(\tilde\phi_T(\tau))}{d^D(\tilde\phi_T(\tau))}$, and $\tilde\phi_T(\tau^*)$ is when achieves the maximum. It's obvious that $\forall \tau_h$ such that $\tilde\phi_T(\tau_h)=\tilde\phi_T(\tau^*)$, $w^\star(\tau_h)\leq w^\star(\tau_h^\star)$. Denote $\tau_h':=\argmax_{\tilde\phi_T(\tau_h)=\tilde\phi_T(\tau^*)}w^\star(\tau_h)$, then $w^*(\tilde\phi_T(\tau^*))=\frac{\sum_{\tilde\phi_T(\tau_h)=\tilde\phi_T(\tau^*)}d^{[\pi_e^\phi]_{\rm true}}(\tau_h)}{\sum_{\tilde\phi_T(\tau_h)=\tilde\phi_T(\tau^*)}d^{D}(\tau_h)}$.
    Notice that
    \begin{align*}
        \frac{\sum_{\tilde\phi_T(\tau_h)=\tilde\phi_T(\tau^*)}d^{[\pi_e^\phi]_{\rm true}}(\tau_h)}{\sum_{\tilde\phi_T(\tau_h)=\tilde\phi_T(\tau^*)}d^{D}(\tau_h)}\leq \frac{\sum_{\tilde\phi_T(\tau_h)=\tilde\phi_T(\tau^*)}d^{D}(\tau_h)\cdot d^{[\pi_e^\phi]_{\rm true}}(\tau_h')/d^D(\tau_h')}{\sum_{\tilde\phi_T(\tau_h)=\tilde\phi_T(\tau^*)}d^{D}(\tau_h)}=w^\star(\tau_h')
    \end{align*}
    Consequently, $w^*(\tilde\phi_T(\tau^*))\leq w^\star(\tau_h')\leq w^\star(\tau_h^\star)$, which prove the theorem.
\end{proof}

\section{Relation With Deep Abstraction}
In this section, we discuss our relation with OPE methods that explicitly construct abstractions. Our method uses abstraction purely as a tool for analysis: we analyze existing algorithms without changing them, but simply reveal when and how these algorithms admit improved guarantees due to belief-space smoothness. On the contrary, some other methods actively construct an abstraction to simplify OPE. In such cases, algorithms running on an abstract system (thus simpler than the original system) may achieve smaller error guarantees. In this section, we briefly compare our idea and that of deep abstraction~\cite{hao2024off}.

In~\cite{hao2024off}, they designed a method that construct a deep abstraction in MDPs using the conventional abstraction techniques, by applying two different methods of abstraction recursively to obtain a deep abstraction. And provably, the variance of the abstracted system monotonously decreases as the abstraction goes deeper.

\paragraph{Comparison with our settings.}
\begin{enumerate}
    \item \textbf{Differences in type and strictness of abstraction:} The abstraction in this paper requires, at each step, an either forward-model-irrelevant condition or backward-model-irrelevant condition. As we know, bisimulation, whether or not in its approximate version, is a very strong condition to fulfill, and becomes especially restrictive in belief spaces with exponential cardinality and limited structure. Also, since it’s using conventional abstraction skills, it doesn’t require the metric structure of the state space. In contrast, our abstraction is based on an $\varepsilon$-net over the belief space, which leverages the metric geometry of the space and applies uniformly to a wide range of POMDPs regardless of structural assumptions.
    \item \textbf{In solving the curse of horizon:}  Indeed,~\cite{hao2024off} elegantly showed that the MSE monotonously decreases as the abstraction goes deeper. But to address the curse of memory/horizon via abstraction, one must analyze how coverage improves in the abstract space. Notably, directly applying their analysis to POMDPs reveals that Assumptions 2 and 4 implicitly hide an exponential constant within $O(1)$
. This constant stems from the boundedness of the function class $\CAL W$
, which includes the MIS ratio $\hat w^\pi$
 and is assumed finite under Assumption 2. While this is acceptable in MDPs where no curse of horizon/memory exists, in POMDPs, it is crucial to account for how abstraction influences this exponential term.
\end{enumerate}

\section{Future Algorithm Design}
While our paper focuses on the theoretical framework, the stability perspective of our analysis naturally inspires concrete algorithmic ideas for future work.
\begin{enumerate}
    \item \textbf{Stability-regularized training:} Augment Bellman-error minimization or value-function fitting with an additional penalty term
    \begin{align*}
        \lambda\mathbb{E}_{\CAL D^{\otimes 2}}[\mathbb{I}(\|\hat b_1-\hat b_2\|_1\leq \epsilon)\cdot |V(\hat b_1)-V(\hat b_2)|].
    \end{align*}
    \item \textbf{Post-training stability selection:} Train multiple candidate policies, then select the one with the highest empirical stability measured over belief neighborhoods. Theoretically, this is equivalent to the above penalty approach as $\lambda\to 0$.
\end{enumerate}

\section{Limitations}\label{appendix:limitation}
Despite our general result is provably no worse than the original coverage assumption, it is possible in some circumstances that the metric property of belief space cannot improve the coverage either. The simplest scenario to consider is when every history has a unique one-hot belief state, and the POMDP is merely equivalent to a MDP with exponentially large state space. In this case, the belief metric is a discrete metric, for $\forall b_1,b_2\in\CAL{B},b_1\neq b_2\to\|b_1-b_2\|_1=2$, and the covering number is exactly the cardinality of the space, which is exponential. This reveals the limitation of our analysis in cases when belief space is sparse, or when lack of some specific smoothness structure. However, information-theoretically, OPE problems for POMDPs always suffer from the curse of Horizon in the most general case as shown in~\cite{zhang2025statistical}, meaning that structural assumptions or specific properties of the system must be utilized to gain meaningful progress.

Another limitation is when sample size becomes too large comparing to the horizon $H$. Notice that in the finite sample argument provided by our result (e.g. Corollary~\ref{coro:finite_sample_1}), the abstract coverage depends on the approximation level $\varepsilon$, which is set to $O(n^{-1})$. If $n$ becomes too large in this case, the $O(n^{-1})$-covering number will converge to the cardinality of the space $\CAL{B}$ itself, which is exponential w.r.t. the horizon $H$. This also trivialize our analysis. Therefore, when considering finite horizon POMDPs, the horizon should be relatively large comparing to the sample size for our result to be valid.

\nocite{subramanian2022approximate,asadi2018lipschitz}

\end{document}